\documentclass[letterpaper]{article} 
\usepackage{ijcai23}  
\usepackage{times}  
\usepackage{helvet}  
\usepackage{courier}  
\usepackage[hyphens]{url}  
\usepackage{graphicx} 
\usepackage{natbib}  
\usepackage{caption} 
\usepackage[switch]{lineno}

\usepackage[ruled,linesnumbered]{algorithm2e}
\SetKwInput{KwInput}{Input}
\usepackage{amsmath}
\usepackage{amssymb}
\usepackage{amsthm}
\usepackage{booktabs}
\usepackage{calc}
\usepackage{changepage}
\usepackage{multirow}
\usepackage{pifont}
\usepackage{subcaption}
\usepackage{xcolor}

\usepackage{macros}

\title{NeuPSL: Neural Probabilistic Soft Logic}

\author {
    Connor Pryor$^{*1}$ \and
    Charles Dickens$^{*1}$ \and
    Eriq Augustine $^1$ \and
    Alon Albalak $^2$ \and \\
    William Wang $^2$ \And
    Lise Getoor $^1$\\
    \affiliations
    $^1$ UC Santa Cruz \\
    $^2$ UC Santa Barbara\\
    \emails
    cfpryor@ucsc.edu, cadicken@ucsc.edu, eaugusti@ucsc.edu, alon\_albalak@ucsb.edu, william@cs.ucsb.edu, getoor@ucsc.edu
}

\begin{document}
    \maketitle
    \def\thefootnote{*}\footnotetext{These authors contributed equally to this work}

    \begin{abstract}
        In this paper, we introduce \textit{\longname{}} (\shortname{}), a novel neuro-symbolic (NeSy) framework that unites state-of-the-art symbolic reasoning with the low-level perception of deep neural networks.
To model the boundary between neural and symbolic representations, we propose a family of energy-based models, \textit{NeSy Energy-Based Models}, and show that they are general enough to include NeuPSL and many other NeSy approaches.
Using this framework, we show how to seamlessly integrate neural and symbolic parameter learning and inference in \shortname{}.
Through an extensive empirical evaluation, we demonstrate the benefits of using NeSy methods, achieving upwards of 30\% improvement over independent neural network models.
On a well-established NeSy task, MNIST-Addition, \shortname{} demonstrates its joint reasoning capabilities by outperforming existing NeSy approaches by up to 10\% in low-data settings.
Furthermore, \shortname{} achieves a 5\% boost in performance over state-of-the-art NeSy methods in a canonical citation network task with up to a $40$ times speed up.



    \end{abstract}

    \section{Introduction}

The field of artificial intelligence (AI) has long sought a symbiotic union of neural and symbolic methods.  
Neural-based methods excel at low-level perception and learn from large training data sets but struggle with interpretability and generalizing in low-data settings.
Meanwhile, symbolic methods can effectively use domain knowledge, context, and common sense to reason with limited data but have difficulty representing complex low-level patterns. 
Recently, neuro-symbolic computing (NeSy) \citep{besold:arxiv17, garcez:jal19, deraedt:ijcai20} has emerged as a promising new research area with the goal of developing systems that integrate neural and symbolic methods in a mutually beneficial manner.

A neural and symbolic union has the potential to yield two highly desirable capabilities - the ability to perform structured prediction (\textit{joint inference}) across related examples that possess complex low-level features and the ability to jointly learn (\textit{joint learning}) and adapt parameters over neural and symbolic models simultaneously.
For instance, predicting the result of competitions between teams using historical performance statistics in a tournament bracket requires methods to perform joint inference to reason over low-level trends and avoid inconsistencies such as two first-place finishes.
Unfortunately, joint inference problems quickly grow in complexity as the output space typically increases combinatorially.
For example, in the tournament setting, as the number of entries increases, the number of potential solutions grows exponentially ($O(2^{n})$).
An open challenge in the NeSy community is scaling joint inference and reasoning.

This paper introduces \textit{Neural Probabilistic Soft Logic} (\shortname{}), a novel NeSy method that integrates deep neural networks with a symbolic method designed for fast joint learning and inference.
\shortname{} extends probabilistic soft logic (PSL) \citep{bach:jmlr17}, a state-of-the-art and scalable probabilistic programming framework that can reason statistically (using probabilistic inference) and logically (using soft rules).
PSL has been shown to excel in a wide variety of tasks, including natural language processing \citep{beltagy:acl14, deng:emnlp15, liu:aaai16, rospocher:iswc18}, data mining \citep{alshukaili:iswc16, kimmig:tkde19}, recommender systems \citep{kouki:recsys15}, knowledge graph discovery \citep{pujara:iswc13}, fairness modeling \citep{farnadi:deb19, dickens:fatrec20}, and causal reasoning \citep{sridhar:ijcai18}.
The key innovation of \shortname{} is a new class of predicates that rely on neural network output for their values.
This change fundamentally alters the learning and joint inference problems by requiring efficient integrated symbolic and neural parameter learning.
The appeal of this extension is that it allows for the semantics and implementation of the symbolic language to remain the same as PSL, while also incorporating the added benefit of low-level neural perception.
To gain a deeper understanding of optimizing the symbolic and neural parameters, we propose a versatile mathematical framework, Neuro-Symbolic Energy-Based Models (NeSy-EBMs), that enables many NeSy systems to utilize established Energy-Based Model learning losses and algorithms.
Utilizing this theory and leveraging the unique relaxation properties of PSL, we show that a gradient over these neural predicates can be calculated and passed back to common back-propagation engines such as PyTorch or Tensorflow, allowing for scalable end-to-end gradient training.

Our key contributions include:
1) We define \textit{Neuro-Symbolic Energy-Based Models} (NeSy-EBMs), a family of energy-based models, and show how they provide a foundation for describing, understanding and comparing NeSy systems.
2) We introduce \shortname{}, describe how it fits into the NeSy ecosystem and supports scalable joint inference, and 
show how it can be trained end-to-end using a joint energy-based learning loss.
3) We perform extensive evaluations over two image classification tasks and two citation network datasets.
Our results show \shortname{} consistently outperforms existing approaches on joint inference tasks and can more efficiently leverage structure, particularly in low-data settings.

    \section{Related Work}
\label{sec:related-work}

Neuro-symbolic computing (NeSy) is an active area of research that aims to incorporate logic-based reasoning with neural networks
\citep{garcez:book02,bader:wwst05,garcez:book09,serafini:aiia16,besold:arxiv17,donadello:ijcai17,yang:neurips17,evans:jair18,manhaeve:ai21,garcez:jal19,deraedt:ijcai20,lamb:ijcai20,badreddine:ai22}.
The advantages of NeSy systems include interpretability, robustness, and the ability to integrate various sub-problem solutions (such as perception, reasoning, and decision-making).
For a thorough introduction to NeSy literature, we refer the reader to the excellent surveys by \citenoun{besold:arxiv17} and \citenoun{deraedt:ijcai20}.
In this section, we identify key NeSy research categories and provide a brief description of each.

\textbf{Differentiable frameworks of logical reasoning:}
Methods in this category use neural networks' universal function approximation properties to emulate logical reasoning inside networks.
Examples include: \citenoun{rocktaschel:neurips17}, \citenoun{bovsnjak:icml17}, \citenoun{evans:jair18},
and \citenoun{cohen:jair20}.

\textbf{Constrained Output:}
These approaches enforce constraints or regularizations on the output of neural networks.
Examples include: \citenoun{hu:acl16}, \citenoun{diligenti:icmla17}, \citenoun{donadello:ijcai17}, \citenoun{mehta:emnlp18}, \citenoun{xu:icml18}, and \citenoun{nandwani:neurips19}.

\textbf{Executable logic programs:}
These approaches use neural models to build executable logical programs.
Examples include \citenoun{liang:acl17} and \citenoun{mao:iclr19}.
We highlight \ltnlongname{} (\ltnshortname{}) 
\citep{badreddine:ai22}, as we include this approach in our empirical evaluation.
\ltnshortname{} connect neural predictions into functions representing symbolic relations with real-valued or fuzzy logic semantics.

\textbf{Neural networks as predicates:}
This line of work integrates neural networks and probabilistic reasoning by introducing neural networks as predicates in the logical formulae.
This technique provides a very general and flexible framework for NeSy reasoning and allows for the use of multiple networks as well as the full incorporation of constraints and relational information.
Examples include  DASL \citep{sikka:techreport20}, NeurASP \citep{yang:ijcai20}, Nuts\&Bolts \citep{sachan:neurips18}, \dpllongname{} (\dplshortname{}) \citep{manhaeve:ai21}, and our proposed method (\longname{}).
\dplshortname{} combines general-purpose neural networks with the probabilistic modeling of ProbLog \citep{deraedt:ijcai07} in a way that allows for learning and inference over complex tasks, such as program induction.
We include \dplshortname{} in our empirical evaluation.

    \section{Neuro-Symbolic Energy-Based Models}
\label{sec:nesy-ebms}

With the success and growth of NeSy research, there is an increasing need for a common formalization of NeSy systems to accelerate the research and understanding of the field.
We fill this need with a general mathematical framework, \textit{Neuro-Symbolic Energy-Based Models} (NeSy-EBMs).
NeSy-EBMs encompass previous approaches and establishes the foundation of our approach.
Energy-Based Models (EBMs) \citep{lecun:book06} measure the compatibility of a collection of observed (or input) variables $\mathbf{x} \in \mathcal{X}$ and target (or output) variables $\mathbf{y} \in \mathcal{Y}$ with a scalar-valued \emph{energy function}: $E: \mathcal{Y} \times \mathcal{X} \to \mathbb{R}$.
Low energy states of the variables represent high compatibility.
Prediction or \emph{inference} in EBMs is performed by finding the lowest energy state of the variables $\mathbf{y}$ given $\mathbf{x}$.
Energy functions are parameterized by variables $\mathbf{w} \in \mathcal{W}$, and \emph{learning} is the task of finding a parameter setting that associates low energy to correct solutions.

Building on the well-known EBM framework, NeSy-EBMs are a family of EBMs that integrate neural architectures with explicit encodings of symbolic relations. 
The input variables are organized into neural, $\mathbf{x}_{nn} \in \mathcal{X}_{nn}$, and symbolic, $\mathbf{x}_{sy} \in \mathcal{X}_{sy}$, vectors.
Furthermore, the parameters of the energy function, $\mathbf{w}$, are partitioned into neural weights, $\mathbf{w}_{nn} \in \mathcal{W}_{nn}$, and symbolic weights, $\mathbf{w}_{sy} \in \mathcal{W}_{sy}$.
Formally,
\begin{definition}[NeSy-EBM]
    Let $\mathbf{y} \in \mathcal{Y}$ and $\mathbf{x}_{sy}\in \mathcal{X}_{sy}$ be vectors of variables with symbolic interpretations. 
    Let $\mathbf{g}_{nn}$ be neural networks with \textbf{neural weights} $\mathbf{w}_{nn} \in \mathcal{W}_{nn}$ and inputs $\mathbf{x}_{nn} \in \mathcal{X}_{nn}$. 
    A \textbf{symbolic potential} is a function of $\mathbf{y}$, $\mathbf{x}_{sy}$, and $\mathbf{g}_{nn}(\cdot)$ parameterized by \textbf{symbolic weights} $\mathbf{w}_{sy} \in \mathcal{W}_{sy}$: $\psi(\mathbf{y}, \mathbf{x}_{sy}, \mathbf{w}_{sy}, \mathbf{g}_{nn}(\mathbf{x}_{nn}, \mathbf{w}_{nn})) \in \mathbb{R}$.
    A \textbf{NeSy-EBM energy function} is a mapping of a vector of $m$ symbolic potential outputs, $\Psi(\mathbf{y}, \mathbf{x}_{sy}, \mathbf{w}_{sy}, \mathbf{x}_{nn}, \mathbf{w}_{nn}) = [\psi_{i}(\mathbf{y}, \mathbf{x}_{sy}, \mathbf{w}_{sy}, \mathbf{g}_{nn}(\mathbf{x}_{nn}, \mathbf{w}_{nn}))]_{i = 1}^{m}$, to a real value: 
    $E(\Psi(\mathbf{y}, \mathbf{x}_{sy}, \mathbf{w}_{sy}, \mathbf{x}_{nn}, \mathbf{w}_{nn})) \in \mathbb{R}$.
    \label{def:NeSy-EBM}
\end{definition}
NeSy-EBMs are differentiated from one another by the instantiation process, the form of the symbolic potentials, and the definition of the energy function.
In appendix, we formally show how two NeSy systems \dpllongname{} (\dplshortname{}) \citep{manhaeve:neurips18} and \ltnlongname{} (\ltnshortname{}) \citep{badreddine:ai22} fit into the NeSy-EBM framework.
In summary, \dplshortname{} uses neural network outputs to specify event probabilities that are used in logical formulae defining probabilistic dependencies.
The definition of the \dplshortname{} symbolic potentials and energy function are tied to the inference task; a different definition of the symbolic potential and energy function is used to implement marginal versus MAP inference.
For marginal, the most common \dplshortname{} inference, symbolic potentials are functions of marginal probabilities, and the energy function is a joint distribution that is the sum of the symbolic potentials.
\ltnshortname{} instantiate a model which forwards neural network predictions into functions representing symbolic relations with real-valued or fuzzy logic semantics.
The fuzzy logic functions are symbolic potentials that are aggregated to define the energy function.
The following section will introduce how our approach, \shortname{}, is instantiated as a NeSy-EBM.
Using this common framework, understanding and theoretical advances can be made across NeSy approaches. 


\begin{figure*}[t]
    \centering

    \includegraphics[width=0.80 \textwidth]{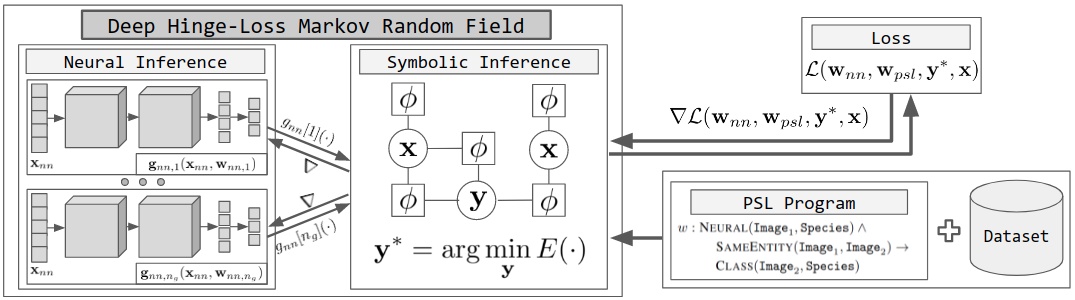}

    \caption{
        \shortname{} inference and learning pipeline.
    }
    \label{fig:neupsl-pipeline}
\end{figure*}

\subsection{Joint Reasoning in  NeSy-EBMs}

We highlight two important categories of NeSy-EBM energy functions: \emph{joint} and \emph{independent}.
Formally, an energy function that is additively separable over the output variables $\mathbf{y}$ is an \textit{independent energy function}, i.e., corresponding to each of the $n_{y}$ components of the output variable $\mathbf{y}$ there exists functions $n_{y}$ functions $E_{1}(\mathbf{y}[1], \mathbf{x}_{sy}, \mathbf{w}_{sy}, \mathbf{g}(\mathbf{x}_{nn}, \mathbf{w}_{nn}))$, $\cdots$, $E_{n_{y}}(\mathbf{y}[n_{y}], \mathbf{x}_{sy}, \mathbf{w}_{sy}, \mathbf{g}(\mathbf{x}_{nn}, \mathbf{w}_{nn}))$ such that 
\begin{align*}
    E(\cdot) = \sum_{i = 1}^{n_{y}} E_{i}(\mathbf{y}[i], \mathbf{x}_{sy}, \mathbf{w}_{sy}, \mathbf{g}(\mathbf{x}_{nn}, \mathbf{w}_{nn})).
\end{align*}
While a function that is not separable over output variables $\mathbf{y}$ is a \textit{joint energy function}.
This categorization allows for an important distinction during inference and learning.
Independent energy functions simplify inference and learning as finding an energy minimizer, $\mathbf{y}^{*}$, can be distributed across the independent functions $E_{i}$.
In other words, the predicted value for a variable $\mathbf{y}[i]$ has no influence over that of $\mathbf{y}[j]$ where $j \neq i$ and can therefore be predicted separately, i.e., independently.
However, independent energy functions cannot leverage some joint information that may be used to improve predictions.
See appendix for further details.

    \section{\longname{}}
\label{sec:neupsl}
Having laid the NeSy-EBM groundwork, we now introduce  \textit{\longname{}} (\shortname{}), a novel NeSy-EBM framework that extends the probabilistic soft logic (PSL) framework \citep{bach:jmlr17}.
At its core, \shortname{} leverages the power of neural networks' low-level perception by seamlessly integrating their outputs with a collection of symbolic potentials generated through a PSL program. 
\figref{fig:neupsl-pipeline} provides a graphical representation of this process.
The symbolic potentials and neural networks together define a \emph{deep hinge-loss Markov random field} (\emph{Deep-HL-MRF}), a tractable probabilistic graphical model that supports scalable convex joint inference.
This section provides a comprehensive description of how \shortname{} instantiates its symbolic potentials and how the symbolic potentials are combined to define an energy function,
while the following section details \shortname{}'s end-to-end neural-symbolic inference, learning, and joint reasoning processes.

\shortname{} instantiates the symbolic potentials of its energy function using the PSL language where dependencies between relations and attributes of entities in a domain, defined as \emph{atoms}, are encoded with weighted first-order logical clauses and linear arithmetic inequalities referred to as \emph{rules}.
To illustrate, consider a setting in which a neural network is used to classify the species of an animal in an image.
Further, suppose there exists external information suggesting when two images may contain the same entity.
The information linking the images may come from various sources,
such as the images' caption or metadata indicating the images were captured by the same device within a short period of time.
\shortname{} represents the neural network's animal classification of an image ($\pslarg{Image}_{1}$) as a species ($\pslarg{Species}$) with the atom $\pslpred{Neural}(\pslarg{Image}_{1}, \pslarg{Species})$ and the probability that two images ($\pslarg{Image}_{1}$ and $\pslarg{Image}_{2}$) contain the same entity with the atom $\pslpred{SameEntity}(\pslarg{Image}_{1}, \pslarg{Image}_{2})$.
Additionally, we represent \shortname{}'s classification of $\pslarg{Image}_{2}$ with $\pslpred{Class}(\pslarg{Image}_{2}, \pslarg{Species})$.
The following weighted logical rule in \shortname{} represents the notion that two images identified as the same entity may also be of the same species:
\begin{align}
    w: \pslpred{Neural}(& \pslarg{Image}_{1}, \pslarg{Species}) \nonumber \\
    & \psland \pslpred{SameEntity}(\pslarg{Image}_{1}, \pslarg{Image}_{2}) \nonumber \\
    & \rightarrow \ \pslpred{Class}(\pslarg{Image}_{2}, \pslarg{Species})
    \label{eq:example_neupsl_rule}
\end{align}
The parameter $ w $ is the weight of the rule, and it quantifies its relative importance in the model.
Note these rules can either be hard or soft constraints.
Atoms and weighted rules are templates for creating symbolic potentials or soft constraints.
To create these symbolic potentials, atoms and rules are instantiated with observed data and neural predictions.
Atoms instantiated with elements from the data are referred to as \emph{ground atoms}.
Then, valid combinations of ground atoms substituted in the rules create \emph{ground rules}.
To illustrate, suppose that there are two images $\{Id1, Id2\}$ and three species classes $\{Cat, Dog, Frog\}$.
Using the above data for cats would result in the following ground rules (analogous ground rules would be created for dogs and frogs):
\begin{align*}
    w: \pslpred{Neural}(Id1, Cat) \psland & \pslpred{SameEntity}(Id1, Id2) \\
    & \rightarrow \ \pslpred{Class}(Id2, Cat) \\
    w: \pslpred{Neural}(Id2, Cat) \psland & \pslpred{SameEntity}(Id2, Id1) \\
    & \rightarrow \ \pslpred{Class}(Id1, Cat)
\end{align*}
Ground atoms are mapped to either an observed variable, $x_{sy,i}$, target variable, $y_{i}$, or a neural function with inputs $\mathbf{x}_{nn}$ and parameters $\mathbf{w}_{nn, i}$: $g_{nn, i}(\mathbf{x}_{nn}, \mathbf{w}_{nn, i})$.
Then, variables are aggregated into the vectors $\mathbf{x}_{sy}=[x_{sy_i}]_{i = 1}^{n_{x}}$ and $\mathbf{y}=[y_{i}]_{i = 1}^{n_y}$ and neural outputs are aggregated into the vector $\mathbf{g}_{nn} = [g_{nn, i}]_{i = 1}^{n_{g}}$.
Ground rules are either logical (e.g., \eqnref{eq:example_neupsl_rule}) or arithmetic defined over $\mathbf{x}_{sy}$, $\mathbf{y}$, and $\mathbf{g}_{nn}$.
These ground rules create one or more potentials $\phi(\cdot) \in \mathcal{R}$, where logical rules are relaxed using Łukasiewicz continuous valued logical semantics \citep{klir:book95}.
\commentout{
Each ground rule creates one or more potentials $\phi_{i}$ defined over $\mathbf{x}_{sy}$, $\mathbf{y}$, and $\mathbf{g}_{nn}$.
Logical rules, such as the one given in example \eqref{eq:example_neupsl_rule}, are relaxed using Łukasiewicz continuous valued logical semantics \citep{klir:book95}.
\shortname{} also supports arithmetic rules for defining penalty functions.
}
Each potential $\phi(\cdot)$ is associated with a weight $w_{psl}$ inherited from its instantiating rule.
The potentials and weights from the instantiation process are used to define a member of a tractable class of graphical models, \emph{deep hinge-loss Markov random fields} (Deep-HL-MRF):

\begin{definition}[Deep Hinge-Loss Markov Random Field]
    Let $\mathbf{y} \in [0, 1]^{n_{y}}$ and $\mathbf{x}_{sy} \in [0, 1]^{n_{x}}$ be vectors of $[0, 1]$ valued variables. 
    Let $\mathbf{g}_{nn} = [g_{nn, i}]_{i = 1}^{n_{g}}$ be functions with corresponding parameters $\mathbf{w}_{nn}=[\mathbf{w}_{nn, i}]_{i = 1}^{n_{g}}$ and inputs $\mathbf{x}_{nn}$.
    A \textbf{deep hinge-loss potential} is a function of the form

    \begin{align}
        \phi(\mathbf{y}, \mathbf{x}_{sy}, \mathbf{x}_{nn}, \mathbf{w}_{nn}) = \max(l(\mathbf{y}, \mathbf{x}_{sy}, \mathbf{g}_{nn}(\mathbf{x}_{nn}, \mathbf{w}_{nn})), 0)^\alpha
        \label{eq:neupsl-potential}
    \end{align}
    where $l(\cdot)$ is a linear function and $\alpha \in \{1, 2\}$.
    Let $\mathcal{T} = [t_i]_{i = 1}^{r}$ denote an ordered partition of a set of $m$ deep hinge-loss potentials: $\{\phi_{1}, \cdots, \phi_{m}\}$. 
    For each partition $t_{i}$ define $\Phi_i(\mathbf{y} , \mathbf{x}_{sy}, \mathbf{x}_{nn}, \mathbf{w}_{nn}) := \sum_{j \in t_i} \phi_{i}(\mathbf{y}, \mathbf{x}_{sy}, \mathbf{x}_{nn}, \mathbf{w}_{nn})$ and let $\mathbf{\Phi(\mathbf{y} , \mathbf{x}_{sy}, \mathbf{x}_{nn}, \mathbf{w}_{nn})} := [\Phi_i(\mathbf{y} , \mathbf{x}_{sy}, \mathbf{x}_{nn}, \mathbf{w}_{nn})]_{i = 1}^{r}$.
    Further, let $\mathbf{w}_{psl} = [w_{psl,i}]_{i = 1}^{r}$ be a vector of non-negative weights corresponding to the partition $\mathcal{T}$.
    Then, a \textbf{deep hinge-loss energy function} is
    \begin{align}
    E(\mathbf{y}, \mathbf{x}_{sy}, \mathbf{x}_{nn}, \mathbf{w}_{nn}, \mathbf{w}_{psl}) 
    & = \mathbf{w}_{psl}^T \mathbf{\Phi(\mathbf{y} , \mathbf{x}_{sy}, \mathbf{x}_{nn}, \mathbf{w}_{nn})}
    \label{eq:hl-mrf_energy_function}
    \end{align}
    Further, let $\mathbf{c} = [c_{i}]_{i = 1}^{q}$ be a vector of $q$ linear constraints in standard form, defining the feasible set 
    $\mathbf{\Omega}
        = \left\{
        \mathbf{y}, 
        \mathbf{x}_{sy}\, \vert \, c_{i}(\mathbf{y}, \mathbf{x}_{sy}) \leq 0, \, \forall i \in \{0, \cdots, q\}
        \right\}$.
    Then a \textbf{deep hinge-loss Markov random field}, $\mathcal{P}$, with random variables $\mathbf{y}$ conditioned on $\mathbf{x}_{sy}$ and $\mathbf{x}_{nn}$ is a probability density of the form
    \begin{align*}
        P(\mathbf{y} \vert \mathbf{x}_{sy}, \mathbf{x}_{nn}) = 
        \begin{cases}
        \frac{\exp (-E(\mathbf{y}, \mathbf{x}_{sy}, \mathbf{x}_{nn}, \mathbf{w}_{nn}, \mathbf{w}_{psl}))}{\int_{\mathbf{y} \vert \mathbf{y}, \mathbf{x}_{sy} \in \mathbf{\Omega}} \exp(-E(\cdot)) d\mathbf{y}} & (\mathbf{y}, \mathbf{x}_{sy}) \in \mathbf{\Omega} \\
        0 & o.w.
        \end{cases} 
    \end{align*}
    \label{def:hlmrf}
\end{definition}

Deep-HL-MRFs naturally fit into the NeSy-EBM framework.
The symbolic potentials of deep-HL-MRFs are the aggregated and scaled deep hinge-loss potentials:
\begin{align}
    &\psi_{\shortname{}}(\mathbf{y}, \mathbf{x}_{sy}, \mathbf{w}_{psl}, \mathbf{g}_{nn}(\mathbf{x}_{nn}, \mathbf{w}_{nn})) \nonumber\\
    & = \mathbf{w}_{psl} \Phi(\mathbf{y}, \mathbf{x}_{sy}, \mathbf{x}_{nn}, \mathbf{w}_{nn})
\end{align}
Then the energy function is the sum of symbolic potentials: 
    \begin{align}
    &E_{\shortname{}}(\mathbf{y}, \mathbf{x}_{sy}, \mathbf{x}_{nn}, \mathbf{w}_{nn}, \mathbf{w}_{psl}) \nonumber \\
    & = \sum_{i = 1}^{r} \psi_{\shortname{}, i}(\mathbf{y}, \mathbf{x}_{sy}, \mathbf{w}_{psl}, \mathbf{g}_{nn}(\mathbf{x}_{nn}, \mathbf{w}_{nn}))
    \end{align}
    
    \section{\shortname{} Inference and Learning}
\label{sec:inference_and_learning}

There is a clear connection between neural and symbolic inference in \shortname{} that allows any neural architecture to interact with symbolic reasoning in a simple and expressive manner.
The \shortname{} neural-symbolic interface and inference pipeline is shown in \figref{fig:neupsl-pipeline}.
\emph{Neural inference} is computing the output of the neural networks given the input $\mathbf{x}_{nn}$, i.e., computing $g_{nn, i}(\mathbf{x}_{nn}, \mathbf{w}_{nn, i})$ for all $i$.
\shortname{} \emph{symbolic inference} minimizes the energy function over $\mathbf{y}$:
\begin{align}
    \mathbf{y}^* = \argmin_{\mathbf{y} \vert (\mathbf{y}, \mathbf{x}_{sy}) \in \mathbf{\Omega}} E(\mathbf{y}, \mathbf{x}_{sy}, \mathbf{x}_{nn}, \mathbf{w}_{nn}, \mathbf{w}_{psl})
    \label{eq:neupsl_inference}
\end{align}
Note that the hinge-loss potentials are convex in $\mathbf{y}$ and hence, with the common constraint enforcing symbolic parameters to be non-negative, i.e., $\mathbf{w}_{psl} > 0$, the energy function is convex in $\mathbf{y}$.
Any scalable convex optimizer can be applied to solve \eqref{eq:neupsl_inference}.
\shortname{} uses the alternating direction method of multipliers \citep{boyd:ftml10}.

\shortname{} learning is the task of finding both neural and symbolic parameters, i.e., rule weights, that assign low energy to correct values of the output variables and higher energies to incorrect values.
Learning objectives are functionals mapping an energy function and a set of training examples $\mathcal{S} = \{(\mathbf{y}_i, \mathbf{x}_{sy, i}, \mathbf{x}_{nn, i}): i = 1, \cdots, P\}$ to a real-valued loss.
As the energy function for \shortname{} is parameterized by the neural weights $\mathbf{w}_{nn}$ and symbolic weights $\mathbf{w}_{psl}$, we express the learning objective as a function of $\mathbf{w}_{nn}$, $\mathbf{w}_{psl}$, and $\mathcal{S}$: $\mathcal{L}(\mathcal{S}, \mathbf{w}_{nn}, \mathbf{w}_{psl})$.
Learning objectives follow the standard empirical risk minimization framework and are therefore separable over the training examples in $\mathcal{S}$ as a sum of per-sample loss functions $L_{i}(\mathbf{y}_{i}, \mathbf{x}_{i}, \mathbf{x}_{nn, i}, \mathbf{w}_{nn}, \mathbf{w}_{psl})$.
Concisely, \shortname{} learning is the following minimization:
\begin{align*}
    &\argmin_{\mathbf{w}_{nn}, \mathbf{w}_{psl}} \mathcal{L} (\mathbf{w}_{nn}, \mathbf{w}_{psl}, \mathcal{S}) \\
    & \quad = \argmin_{\mathbf{w}_{nn}, \mathbf{w}_{psl}} \sum_{i = 1}^{P} L_{i}(\mathbf{y}_i, \mathbf{x}_{sy, i}, \mathbf{x}_{nn, i}, \mathbf{w}_{nn}, \mathbf{w}_{psl})
\end{align*}
In the learning setting, variables $\mathbf{y}_{i}$ from the training set $S$ are partitioned into vectors $\mathbf{y}_{i, t}$ and $\mathbf{z}_{i}$.
The variables $\mathbf{y}_{i, t}$ represent variables for which there is a corresponding truth value, while $\mathbf{z}_{i}$ represent latent variables.
Without loss of generality, we write $\mathbf{y}_{i} = (\mathbf{y}_{i, t}, \mathbf{z}_{i})$.

There are multiple losses that one could motivate for optimizing the parameters of an EBM.
Common losses, including the loss we present in this work, use the following terms:
\begin{align*}
    & \mathbf{z}_{i}^* = \argmin_{\mathbf{z} \vert ((\mathbf{y}_{i, t}, \mathbf{z}), \mathbf{x}) \in \mathbf{\Omega}}E((\mathbf{y}_{i, t}, \mathbf{z}), \mathbf{x}_{sy, i}, \mathbf{x}_{nn, i}, \mathbf{w}_{nn}, \mathbf{w}_{psl}) \\
    & \mathbf{y}_{i}^{*} = \argmin_{\mathbf{y} \vert (\mathbf{y}, \mathbf{x}_{i}) \in \mathbf{\Omega}} E(\mathbf{y}, \mathbf{x}_{sy, i}, \mathbf{x}_{nn, i}, \mathbf{w}_{nn}, \mathbf{w}_{psl})
\end{align*}
In words, $\mathbf{z}_{i}^{*}$ and $\mathbf{y}_{i}^{*}$ are the lowest energy states given $(\mathbf{y}_{i, t}, \mathbf{x}_{sy, i}, \mathbf{x}_{nn, i})$ and $(\mathbf{x}_{sy, i}, \mathbf{x}_{nn, i})$, respectively.
A special case of learning is when the per-sample losses are not functions of $\mathbf{z}_{i}^{*}$ and $\mathbf{y}_{i}^{*}$, and more specifically, the losses do not require any subproblem optimization.
We refer to this situation as \textit{constraint learning}.
Constraint learning reduces the time required per iteration at the cost of expressivity.

All interesting learning losses for \shortname{} are a composition of the energy function.
Thus, a gradient-based learning algorithm will require the following partial derivatives: 
\footnote{
    Note arguments of the energy function and symbolic potentials are dropped for simplicity, i.e.,
    $E(\cdot) = E(\mathbf{y}, \mathbf{x}_{sy, i}, \mathbf{x}_{nn, i}, \mathbf{w}_{nn}, \mathbf{w}_{psl})$, 
    $\phi(\cdot) = \phi(\mathbf{y}, \mathbf{x}_{sy}, \mathbf{x}_{nn}, \mathbf{w}_{nn}),$ and $\mathbf{g}_{nn}[i](\cdot) = \mathbf{g}_{nn}[i](\mathbf{x}_{nn}, \mathbf{w}_{nn})$.
}
\begin{align*}
    & \frac{\partial E(\cdot)}{\partial \mathbf{w}_{psl}[i]} = \Phi_{i}(\mathbf{y}, \mathbf{x}_{sy}, \mathbf{x}_{nn}, \mathbf{w}_{nn}) \\
    & \frac{\partial E(\cdot)}{\partial \mathbf{w}_{nn}[i]} = \mathbf{w}_{psl}^T \nabla_{\mathbf{w}_{nn}[i]} \Phi(\mathbf{y}, \mathbf{x}_{sy}, \mathbf{x}_{nn}, \mathbf{w}_{nn})
\end{align*}
Continuing with the derivative chain rule and noting the potential can be squared ($\alpha = 2$) or linear ($\alpha = 1$), the potential partial derivative with respect to $\mathbf{w}_{nn}[i]$ is the piece-wise defined function:\footnotemark[1]
\begin{align*}
    \frac{\partial \phi(\cdot)}{\partial \mathbf{w}_{nn}[i]} &= 
    \begin{cases}
        \frac{\partial}{\partial 
        \mathbf{g}_{nn}[i]} \phi(\cdot) \cdot \frac{\partial}{\partial \mathbf{w}_{nn}[i]} \mathbf{g}_{nn}[i](\cdot)
        & \alpha = 1 \\
        2 \cdot \phi(\cdot) \cdot \frac{\partial}{\partial \mathbf{g}_{nn}[i]} \phi(\cdot) \cdot \frac{\partial}{\partial \mathbf{w}_{nn}[i]} \mathbf{g}_{nn}[i](\cdot)
        & \alpha = 2\\
    \end{cases}\\
    \frac{\partial \phi(\cdot)}{\partial \mathbf{g}_{nn}[i]} &= 
    \begin{cases}
        0 & \phi(\cdot) = 0 \\
        \frac{\partial}{\partial g_{nn}[i]} l(\mathbf{y}, \mathbf{x}_{sy}, \mathbf{g}_{nn}(\mathbf{x}_{nn}, \mathbf{w}_{nn})) 
        & \phi(\cdot) > 0\\
    \end{cases}
\end{align*}
Since $l(\mathbf{y}, \mathbf{x}_{sy}, \mathbf{g}_{nn}(\mathbf{x}_{nn}, \mathbf{w}_{nn}))$ is a linear function, the partial gradient with respect to $\mathbf{g}_{nn}[i]$ is trivial.
With the partial derivatives presented here, standard backpropagation-based algorithms for computing gradients can be applied for both neural and symbolic parameter learning.

\textbf{Energy Loss}:
A variety of differentiable loss functions can be chosen for $\mathcal{L}$.
For simplicity, in this work, we present the \textit{energy loss}.
The energy loss parameter learning scheme directly minimizes the energy of the training samples, i.e., the per-sample losses are:
\begin{align*}
    L_{i}(\mathbf{y}_{i}, \mathbf{x}_{sy, i}, \mathbf{x}_{nn, i}, & \mathbf{w}_{nn}, \mathbf{w}_{psl}) \\
    & = E((\mathbf{y}_{i, t}, \mathbf{z}^{*}_{i}), \mathbf{x}_{sy, i}, \mathbf{x}_{nn, i}, \mathbf{w}_{nn}, \mathbf{w}_{psl})
\end{align*}
Notice that inference over the latent variables is necessary for gradient and objective value computations.
However, a complete prediction from \shortname{}, i.e., inference over all components of $\mathbf{y}$, is unnecessary.
Therefore the parameter learning problem is as follows:
\begin{align*}
    \argmin_{\mathbf{w}_{nn}, \mathbf{w}_{psl}} \sum_{i = 1}^{P} \min_{\mathbf{z} \in \mathbf{\Omega}} \mathbf{w}_{psl}^T \Phi((\mathbf{y}_{i, t}, \mathbf{z}), \mathbf{x}_{sy, i}, \mathbf{x}_{nn, i}, \mathbf{w}_{nn})
\end{align*}
With L2 regularization, the \shortname{} energy function is strongly convex in all components of $\mathbf{y}_{i}$.
Thus, by \citenoun{danskin:siam66}, the gradient of the energy loss, $L_{i}(\cdot)$, with respect to $\mathbf{w}_{psl}$ at $\mathbf{y}_{i}, \mathbf{x}_{i}, \mathbf{x}_{nn, i} \mathbf{w}_{nn}$ is:
\begin{align*}
    \nabla_{\mathbf{w}_{psl}} L_{i}(\mathbf{y}_i, \mathbf{x}_{sy, i}, & \mathbf{w}_{nn}, \mathbf{w}_{psl}) \\
    & = \Phi((\mathbf{y}_{i, t}, \mathbf{z}_{i}^*), \mathbf{x}_{sy, i}, \mathbf{x}_{nn, i}, \mathbf{w}_{nn})
\end{align*}
Then the per-sample energy loss partial derivative with respect to $\mathbf{w}_{nn}[j]$ at $\mathbf{y}_{i}, \mathbf{x}_{sy, i}, \mathbf{x}_{nn, i}, \mathbf{w}_{psl}$ is:
\begin{align*}
    & \frac{\partial L_{i}(\mathbf{y}_i, \mathbf{x}_{sy, i}, \mathbf{x}_{nn, i}, \mathbf{w}_{nn}, \mathbf{w}_{psl})}{\partial \mathbf{w}_{nn}[j]} \\
    & = 
    \sum_{r = 1}^{R} \mathbf{w}_{psl}[r] \sum_{q \in \tau_r} \frac{\partial \phi_{q}((\mathbf{y}_{i, t}, \mathbf{z}_{i}^*), \mathbf{x}_{sy, i}, \mathbf{x}_{nn, i}, \mathbf{w}_{nn})}{\partial \mathbf{w}_{nn}[j]}
\end{align*}
Details on the learning algorithms and accounting for degenerate solutions of the energy loss are included in supplementary materials.
    
    \begin{figure*}[t]
    \centering

    \includegraphics[width=0.65 \textwidth]{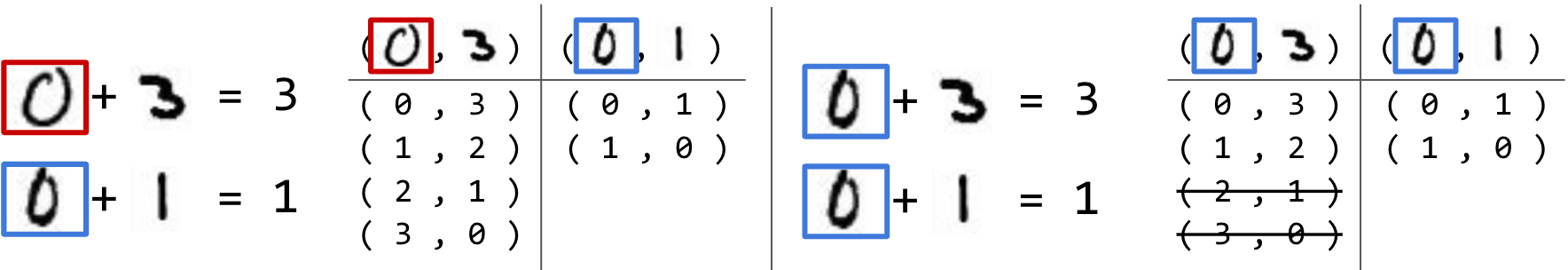}

    \caption{
        Example of overlapping MNIST images in \textbf{\additionlongname{}1}. 
        On the left, distinct images are used for each zero. On the right, the same image is used for both zeros.
    }
    \label{fig:overlapping-digits-example}
\end{figure*}

\begin{figure*}[t]
    \centering
    \begin{subfigure}[b]{\textwidth}
        \centering
        \includegraphics[width=\textwidth]{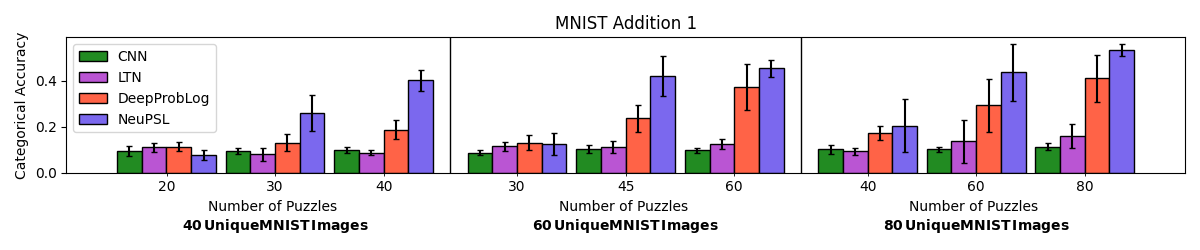}
    \end{subfigure}
    \begin{subfigure}[b]{\textwidth}
        \centering
        \includegraphics[width=\textwidth]{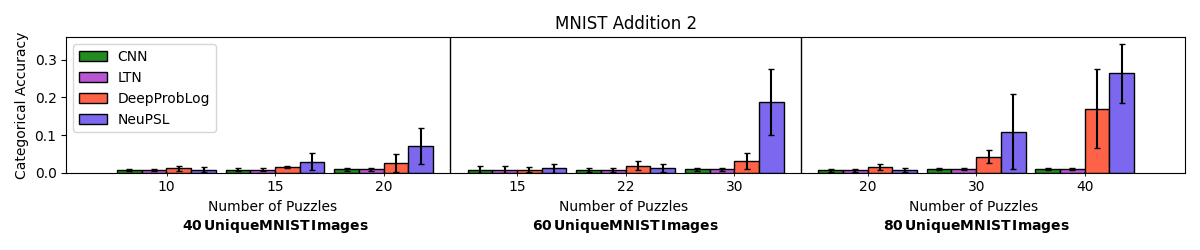}
    \end{subfigure}
    \caption{
        Average test set accuracy and standard deviation on \textbf{\additionlongname{}} datasets with varying amounts of overlap.
    }
    \label{fig:mnist_addition_results}
\end{figure*}

\section{Experimental Evaluation}
\label{sec:evaluation}

We evaluate \shortname{}'s prediction performance and inference time on three tasks to demonstrate the significance of joint symbolic inference and learning.
\shortname{}, implemented using the open-source PSL software package, can be integrated with any neural network library 
(here, we used TensorFlow).\footnote{Implementation details, hyperparameters, network architectures, hardware, and \shortname{} models, are described in the Appendix.

Code and Data: \url{https://github.com/linqs/neupsl-ijcai23}

Appendix: \url{https://arxiv.org/abs/2205.14268}
}
Our investigation addresses the following questions:
Q1) Can neuro-symbolic methods provide a boost over conventional purely data-driven neural models?
Q2) Can we effectively leverage structural relationships across training examples through joint reasoning?
Q3) How does \shortname{} compare with other neuro-symbolic methods in terms of time efficiency on large scale problems?

\subsection{MNIST Addition}
\label{sec:experiments-mnist-addition}
The first set of experiments are conducted on a variation of MNIST Addition, a widely used NeSy evaluation task \citep{manhaeve:neurips18}.
The task involves determining the sum of two lists of MNIST images.
For example, a \textbf{\additionlongname{}1} addition is $(\big[\inlinegraphics{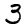}\big] + \big[\inlinegraphics{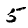}\big] = \mathbf{8})$,
and a \textbf{\additionlongname{}2} addition is
$(\big[\inlinegraphics{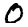}, \inlinegraphics{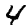}\big] + \big[\inlinegraphics{images/MNIST-3.png}, \inlinegraphics{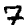}\big] = \mathbf{41})$.
The challenge stems from the lack of labels for the MNIST images in the addition equation.
Only the final sum of the equation is given, leaving the task of identifying the individual digits and determining their values up to the model being used.

While \shortname{} proves to be successful in the original \textbf{\additionlongname{}} setting 
(appendix for further details), here we are interested in exploring the power of joint inference and learning capabilities in NeSy systems.
We introduce a variant of the \textbf{\additionlongname{}} task in which digits are reused across multiple addition examples, i.e., we introduce \emph{overlap}.
\figref{fig:overlapping-digits-example} demonstrates the process of introducing overlap and how joint models narrow the space of possible labels when MNIST images are re-used.
For instance, in the scenario presented in \figref{fig:overlapping-digits-example}, the same MNIST image of a zero is utilized in two separate additions.
To comply with both addition constraints, the potential label space is restricted and can no longer include options such as two or three, as they would violate one of the addition rules.
In contrast, a model performing independent reasoning would have no way of enforcing this constraint across examples.

In the overlap variant of \textbf{\additionlongname{}}, we focus on low-data settings to understand whether NeSy systems' joint reasoning can effectively leverage additional structure to overcome a lack of data.
To create overlap, we begin with a set of $n$ unique MNIST images from which we re-sample to create $(n + m) / 2$ \textbf{\additionlongname{}1} and $(n + m) / 4$ \textbf{\additionlongname{}2} additions.
We vary the amount of overlap with $m \in \{0, n/2, n\}$ and compare performance with $n \in \{40, 60, 80\}$.
Results are reported over ten test sets of $1,000$ MNIST images with overlap proportional to the respective train set.

\figref{fig:mnist_addition_results} summarizes average performance for varying overlap settings.
Each panel varies the number of additions for a set number of unique MNIST images.
For example, the upper left panel presents the results obtained for \additionlongname{}1 with 40 unique images used to generate 20, 30, and 40 additions.
Initially, there is not enough structure from the additions with no overlap  for symbolic inference to discern the correct digit labels for training the neural models.
Then, despite the number of unique MNIST images remaining the same, as the number of additions increases, \dplshortname{} and \shortname{} improve their prediction performance by leveraging the added joint information (Q2).
In all cases, \shortname{} performs best and uses the added structure most efficiently.
\ltnshortname{} and the CNN baseline benefit the least from joint information, a consequence of both learning and inference being performed independently across batches of additions (Q1).

\begin{figure*}[t]
    \centering
    \includegraphics[width=0.95\textwidth]{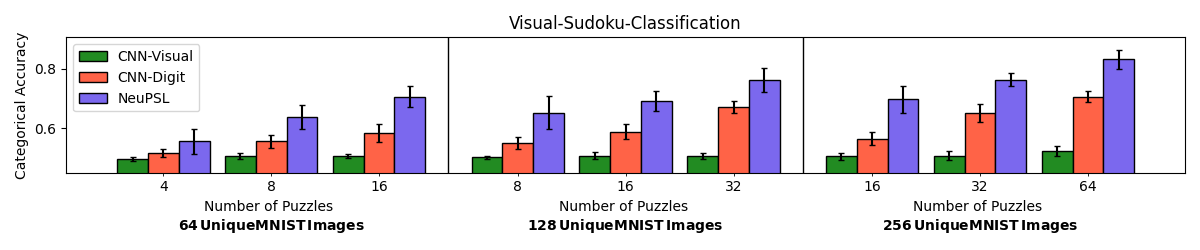}
    \caption{
        Average test set accuracy and standard deviation on \textbf{\sudokulongname{}} with varying amounts of overlap.
    }
    \label{fig:visual_overlap_results}
\end{figure*}

\begin{table*}[t]
    \centering
    \begin{tabular}{c|cc|cc}
        \toprule
            \multirow{2}{*}{Method} & \multicolumn{2}{c}{Citeseer} & \multicolumn{2}{c}{Cora} \\
            & (Accuracy) & (Seconds) & (Accuracy) & (Seconds) \\
        \midrule
            Neural$_{PSL}$      & 57.76 ± 1.71 & - & 57.12 ± 2.13 & - \\
            LP$_{PSL}$          & 50.88 ± 1.18 & - & 73.32 ± 2.39 & - \\
        \midrule
            DeepProbLog         & timeout & timeout & timeout & timeout \\
            DeepStochLog        & 61.30 ± 1.44 & 34.42 ± 0.87 & 69.96 ± 1.47 & 165.28 ± 4.49 \\
            GCN                 & 67.50 ± 0.57 & 3.10 ± 0.04 & 79.52 ± 1.13 & 1.31 ± 0.01 \\
            \citationlp{}       & 67.34 ± 1.17 & 3.98 ± 0.05  & 76.80 ± 2.27 & 4.00 ± 0.31 \\
            \citationfs{}       & \textbf{68.48 ± 1.22} & 4.23 ± 0.05 & \textbf{81.22 ± 0.79} & 4.07 ± 0.14 \\
        \bottomrule
    \end{tabular}
    
    \caption{Test set accuracy and inference runtime in seconds on two citation network datasets.}
    \label{tab:citation-results}
\end{table*}

\subsection{Visual Sudoku Classification}
\label{sec:visual-sudoku}

Inspired by the Visual Sudoku problem proposed by \citenoun{wang:icml19}, \citenoun{augustine:nesy22} introduced a novel NeSy task, \textbf{\sudokulongname{}}. 
In this task, 4x4 Sudoku puzzles are constructed using unlabeled MNIST images. 
The model must identify whether a puzzle is correct, i.e., no duplicate digits in any row, column, or square.
Therefore this task does not require learning the underlying label for images but rather whether an entire puzzle is valid.
For instance, $\big[\inlinegraphics{images/MNIST-3.png}\big]$ does not need to belong to a $"3"$ class, instead $\big[\inlinegraphics{images/MNIST-3.png}\big]$ and $\big[\inlinegraphics{images/MNIST-4.png}\big]$ need to be labeled as different symbols.
Similar to MNIST-Add we explore an overlap variant in low-data settings, with overlapping MNIST images across puzzles.

We compare \shortname{} with two baselines, \vsBaselineVisual{} and \vsBaselineDigit{}.
The first, \vsBaselineVisual{}, takes the pixels for a Sudoku puzzle as input and outputs the probability the puzzle is valid.
The second, \vsBaselineDigit{}, is provided the (unfair) advantage of all sixteen image labels as input.
We use this to verify whether a neural model can learn Sudoku rules.
Scalably developing LTN and DPL models in this new setting is not straightforward due to the large dimensionality of the output space.
A non-expert implementation of a visual sudoku model in DPL and LTN may result in suboptimal reports on model performance and are therefore not included.

\figref{fig:visual_overlap_results} shows the accuracy of \shortname{} and CNN models on \textbf{\sudokulongname{}} with varying amounts of overlap.
\vsBaselineVisual{} and \vsBaselineDigit{} struggle to leverage the problem structure and fail to generalize even the highest data and overlap setting with 256 MNIST images across $64$ puzzles.
However, \shortname{} achieves $70$\% accuracy using roughly $64$ MNIST images across $16$ puzzles, again showing it efficiently leverages joint information across training examples (Q1 and Q2).
This is a particularly impressive result as the neural network in the \shortname{} model was trained to be a $93$\% 4-digit distinguisher without digit labels.

\subsection{Citation Network Node Classification}
\label{sec:experiments-citation-network-node-classification}
In our final experiment, we evaluate the performance of \shortname{} on two widely studied citation network node classification datasets: Citeseer and Cora \citep{sen:aim08}.
In these datasets, symbolic models have the potential to improve predictions by leveraging the homophilic structure of the citation network, i.e., two papers connected in the network are more likely to have the same label. 
This setting differs from \textbf{\sudokulongname{}} and \textbf{\additionlongname{}} as the symbolic relations are not always true.
Moreover, the symbolic relations can be defined over a general and potentially large number of nodes in the network, i.e., a node can be connected to any number of neighbors.

We propose two \shortname{} models for citation network node classification.
Both models integrate a neural network that uses a paper's features to provide an initial classification, which is then adjusted via symbolic reasoning.
The first model, \citationlp{} (Label Propagation), directly uses the bag-of-words feature vector, while the second model, \citationfs{} (Label + Feature Propagation), first performs the feature construction procedure as described in \citenoun{wu:icml19} to obtain a richer representation to provide to the neural model.
We examine the runtime and model performance of NeSy methods \citationlp{}, \citationfs{}, DPL and its scalable extension, DeepStochLog \citep{winters:aaai22}, and a Graph Convolutional Network (GCN) \citep{kipf:iclr17}. 
Additionally, we include the performance of two baselines, LP$_{PSL}$ and Neural$_{PSL}$.
These baselines represent the distinct symbolic and neural components used in the NeuPSL$_{LP}$ model but perform only neural or symbolic reasoning, not both.
We averaged the results over ten randomly sampled splits using 5\% of the nodes for training, 5\% of the nodes for validation, and 1000 nodes for testing.

\tabref{tab:citation-results} shows DeepStochLog, GCN, and NeuPSL all outperform the independent baselines (Q1), with \citationfs{} performing the best. 
These results demonstrate the power of using NeSy systems to effectively leverage structure to improve prediction performance.
Additionally, \shortname{} is capable of scaling its joint inference process to larger structures, achieving higher accuracy with an $8$ and $40$ times speed up over DeepStochLog in Citeseer and Cora, respectively (Q3).
Surprisingly, \shortname{} also achieves a higher prediction performance than even a GCN model while using significantly fewer trainable parameters.

    \section{Conclusion}
In this paper, we introduced \shortname{}, a novel NeSy framework that integrates neural architectures and a tractable class of graphical models for jointly reasoning over symbolic relations and showed its utility across a range of neuro-symbolic tasks.
There are many avenues for future work, including exploring different learning objectives, such as ones that balance traditional neural and energy-based losses and new application domains. 
Each of these is likely to provide new challenges and insights.

    \section*{Acknowledgments}
This work was partially supported by the National Science Foundation grant CCF-2023495.

    \bibliographystyle{named}
    \bibliography{ijcai23}

\begin{thebibliography}{}

\bibitem[\protect\citeauthoryear{Alshukaili \bgroup \em et al.\egroup
  }{2016}]{alshukaili:iswc16}
Duhai Alshukaili, Alvarao Fernandees, and Norman Paton.
\newblock Structuring linked data search results using probabilistic soft
  logic.
\newblock In {\em ISWC}, 2016.

\bibitem[\protect\citeauthoryear{Augustine \bgroup \em et al.\egroup
  }{2022}]{augustine:nesy22}
Eriq Augustine, Connor Pryor, Charles Dickens, Jay Pujara, William~Yang Wang,
  and Lise Getoor.
\newblock Visual sudoku puzzle classification: A suite of collective
  neuro-symbolic tasks.
\newblock In {\em International Workshop on Neural-Symbolic Learning and
  Reasoning (NeSy)}, 2022.

\bibitem[\protect\citeauthoryear{Bach \bgroup \em et al.\egroup
  }{2017}]{bach:jmlr17}
Stephen Bach, Matthias Broecheler, Bert Huang, and Lise Getoor.
\newblock Hinge-loss {M}arkov random fields and probabilistic soft logic.
\newblock {\em JMLR}, 18(1):1--67, 2017.

\bibitem[\protect\citeauthoryear{Bader and Hitzler}{2005}]{bader:wwst05}
Sebastian Bader and Pascal Hitzler.
\newblock Dimensions of neural-symbolic integration - {A} structured survey.
\newblock {\em arXiv}, 2005.

\bibitem[\protect\citeauthoryear{Badreddine \bgroup \em et al.\egroup
  }{2022}]{badreddine:ai22}
Samy Badreddine, Artur {d'Avila Garcez}, Luciano Serafini, and Michael
  Spranger.
\newblock Logic tensor networks.
\newblock {\em AI}, 303(4):103649, 2022.

\bibitem[\protect\citeauthoryear{Beltagy \bgroup \em et al.\egroup
  }{2014}]{beltagy:acl14}
Islam Beltagy, Katrin Erk, and Raymond Mooney.
\newblock Probabilistic soft logic for semantic textual similarity.
\newblock In {\em ACL}, 2014.

\bibitem[\protect\citeauthoryear{Besold \bgroup \em et al.\egroup
  }{2017}]{besold:arxiv17}
Tarek~R. Besold, Artur~S. {d'Avila Garcez}, Sebastian Bader, Howard Bowman,
  Pedro~M. Domingos, Pascal Hitzler, Kai{-}Uwe K{\"{u}}hnberger, Lu{\'{\i}}s~C.
  Lamb, Daniel Lowd, Priscila Machado~Vieira Lima, Leo de~Penning, Gadi Pinkas,
  Hoifung Poon, and Gerson Zaverucha.
\newblock Neural-symbolic learning and reasoning: {A} survey and
  interpretation.
\newblock {\em arXiv}, 2017.

\bibitem[\protect\citeauthoryear{Bo{\v{s}}njak \bgroup \em et al.\egroup
  }{2017}]{bovsnjak:icml17}
Matko Bo{\v{s}}njak, Tim Rockt{\"a}schel, Jason Naradowsky, and Sebastian
  Riedel.
\newblock Programming with a differentiable forth interpreter.
\newblock In {\em ICML}, 2017.

\bibitem[\protect\citeauthoryear{Boyd and Vandenberghe}{2004}]{boyd:book04}
Stephen Boyd and Lieven Vandenberghe.
\newblock {\em Convex Optimization}.
\newblock Cambridge University Press, 2004.

\bibitem[\protect\citeauthoryear{Boyd \bgroup \em et al.\egroup
  }{2010}]{boyd:ftml10}
Stephen Boyd, Neal Parikh, Eric Chu, Borja Peleato, and Jonathan Eckstein.
\newblock Distributed optimization and statistical learning via the alternating
  direction method of multipliers.
\newblock {\em Foundations and Trends in Machine Learning}, 3(1):1--122, 2010.

\bibitem[\protect\citeauthoryear{Cohen \bgroup \em et al.\egroup
  }{2020}]{cohen:jair20}
William~W. Cohen, Fan Yang, and Kathryn Mazaitis.
\newblock Tensorlog: {A} probabilistic database implemented using deep-learning
  infrastructure.
\newblock {\em JAIR}, 67:285--325, 2020.

\bibitem[\protect\citeauthoryear{Danskin}{1966}]{danskin:siam66}
John Danskin.
\newblock The theory of max-min, with applications.
\newblock {\em SIAM Journal on Applied Mathematics}, 14(4):641--664, 1966.

\bibitem[\protect\citeauthoryear{{d'Avila Garcez} \bgroup \em et al.\egroup
  }{2002}]{garcez:book02}
Artur~S. {d'Avila Garcez}, Krysia Broda, and Dov~M. Gabbay.
\newblock {\em Neural-Symbolic Learning Systems: Foundations and Applications}.
\newblock Springer, 2002.

\bibitem[\protect\citeauthoryear{{d'Avila Garcez} \bgroup \em et al.\egroup
  }{2009}]{garcez:book09}
Artur~S. {d'Avila Garcez}, Lu{\'{\i}}s~C. Lamb, and Dov~M. Gabbay.
\newblock {\em Neural-Symbolic Cognitive Reasoning}.
\newblock Springer, 2009.

\bibitem[\protect\citeauthoryear{{d'Avila Garcez} \bgroup \em et al.\egroup
  }{2019}]{garcez:jal19}
Artur {d'Avila Garcez}, Marco Gori, Lu{\'{\i}}s~C. Lamb, Luciano Serafini,
  Michael Spranger, and Son~N. Tran.
\newblock Neural-symbolic computing: An effective methodology for principled
  integration of machine learning and reasoning.
\newblock {\em Journal of Applied Logics}, 6(4):611--632, 2019.

\bibitem[\protect\citeauthoryear{{De Raedt} \bgroup \em et al.\egroup
  }{2007}]{deraedt:ijcai07}
Luc {De Raedt}, Angelika Kimmig, and Hannu Toivonen.
\newblock Problog: A probabilistic prolog and its application in link
  discovery.
\newblock In {\em IJCAI}, 2007.

\bibitem[\protect\citeauthoryear{{De Raedt} \bgroup \em et al.\egroup
  }{2020}]{deraedt:ijcai20}
Luc {De Raedt}, Sebastijan Duman{\v{c}}i{\'c}, Robin Manhaeve, and Giuseppe
  Marra.
\newblock From statistical relational to neuro-symbolic artificial
  intelligence.
\newblock In {\em IJCAI}, 2020.

\bibitem[\protect\citeauthoryear{Deng and Wiebe}{2015}]{deng:emnlp15}
Lingjia Deng and Janyce Wiebe.
\newblock Joint prediction for {Entity/Event-LEvel} sentiment analysis using
  probabilistic soft logic models.
\newblock In {\em EMNLP}, 2015.

\bibitem[\protect\citeauthoryear{Dickens \bgroup \em et al.\egroup
  }{2020}]{dickens:fatrec20}
Charles Dickens, Rishika Singh, and Lise Getoor.
\newblock Hyperfair: A soft approach to integrating fairness criteria.
\newblock In {\em FAccTRec}, 2020.

\bibitem[\protect\citeauthoryear{Dickens \bgroup \em et al.\egroup
  }{2022}]{dickens:tpm22}
Charles Dickens, Connor Pryor, Eriq Augustine, Alon Albalak, and Lise Getoor.
\newblock Efficient learning losses for deep hinge-loss markov random fields.
\newblock In {\em Workshop on Tractable Probabilistic Modeling (TPM)},
  Eindhoven, Netherlands, 2022.

\bibitem[\protect\citeauthoryear{Diligenti \bgroup \em et al.\egroup
  }{2017}]{diligenti:icmla17}
Michelangelo Diligenti, Soumali Roychowdhury, and Marco Gori.
\newblock Integrating prior knowledge into deep learning.
\newblock In {\em ICMLA}, 2017.

\bibitem[\protect\citeauthoryear{Donadello \bgroup \em et al.\egroup
  }{2017}]{donadello:ijcai17}
Ivan Donadello, Luciano Serafini, and Artur~S. {d'Avila Garcez}.
\newblock Logic tensor networks for semantic image interpretation.
\newblock In {\em IJCAI}, 2017.

\bibitem[\protect\citeauthoryear{Evans and Grefenstette}{2018}]{evans:jair18}
Richard Evans and Edward Grefenstette.
\newblock Learning explanatory rules from noisy data.
\newblock {\em JAIR}, 61:1--64, 2018.

\bibitem[\protect\citeauthoryear{Farnadi \bgroup \em et al.\egroup
  }{2019}]{farnadi:deb19}
Golnoosh Farnadi, Behrouz Babaki, and Lise Getoor.
\newblock A declarative approach to fairness in relational domains.
\newblock {\em IEEE Data Engineering Bulletin}, 42(3):36--48, 2019.

\bibitem[\protect\citeauthoryear{Hu \bgroup \em et al.\egroup
  }{2016}]{hu:acl16}
Zhiting Hu, Xuezhe Ma, Zhengzhong Liu, Eduard Hovy, and Eric Xing.
\newblock Harnessing deep neural networks with logic rules.
\newblock In {\em ACL}, 2016.

\bibitem[\protect\citeauthoryear{Kimmig \bgroup \em et al.\egroup
  }{2019}]{kimmig:tkde19}
Angelika Kimmig, Alex Memory, Ren{\'{e}}e~J. Miller, and Lise Getoor.
\newblock A collective, probabilistic approach to schema mapping using diverse
  noisy evidence.
\newblock {\em TKDE}, 31(8):1426--1439, 2019.

\bibitem[\protect\citeauthoryear{Kipf and Welling}{2017}]{kipf:iclr17}
Thomas Kipf and Max Welling.
\newblock Semi-supervised classification with graph convolutional networks.
\newblock In {\em ICLR}, 2017.

\bibitem[\protect\citeauthoryear{Kivinen and Warmuth}{1997}]{kivinen:ic97}
Jyrki Kivinen and Manfred~K. Warmuth.
\newblock Exponentiated gradient versus gradient descent for linear predictors.
\newblock {\em Information and Computation}, 132(1):1--63, 1997.

\bibitem[\protect\citeauthoryear{Klir and Yuan}{1995}]{klir:book95}
George~J. Klir and Bo~Yuan.
\newblock {\em Fuzzy Sets and Fuzzy Logic - Theory and Applications}.
\newblock Prentice Hall, 1995.

\bibitem[\protect\citeauthoryear{Kouki \bgroup \em et al.\egroup
  }{2015}]{kouki:recsys15}
Pigi Kouki, Shobeir Fakhraei, James~R. Foulds, Magdalini Eirinaki, and Lise
  Getoor.
\newblock Hyper: {A} flexible and extensible probabilistic framework for hybrid
  recommender systems.
\newblock In {\em RecSys}, 2015.

\bibitem[\protect\citeauthoryear{Lamb \bgroup \em et al.\egroup
  }{2020}]{lamb:ijcai20}
Lu{\'{\i}}s~C. Lamb, Artur {d'Avila Garcez}, Marco Gori, Marcelo O.~R. Prates,
  Pedro H.~C. Avelar, and Moshe~Y. Vardi.
\newblock Graph neural networks meet neural-symbolic computing: {A} survey and
  perspective.
\newblock In {\em IJCAI}, 2020.

\bibitem[\protect\citeauthoryear{LeCun \bgroup \em et al.\egroup
  }{1998}]{lecun:ieee98}
Yann LeCun, L{\'e}on Bottou, Yoshua Bengio, and Patrick Haffner.
\newblock Gradient-based learning applied to document recognition.
\newblock {\em Proceedings of the IEEE}, 86(11):2278--2324, 1998.

\bibitem[\protect\citeauthoryear{LeCun \bgroup \em et al.\egroup
  }{2006}]{lecun:book06}
Yann LeCun, Sumit Chopra, Raia Hadsell, Marc’Aurelio Ranzato, and Fu~Jie
  Huang.
\newblock A tutorial on energy-based learning.
\newblock {\em Predicting Structured Data}, 1(0), 2006.

\bibitem[\protect\citeauthoryear{Liang \bgroup \em et al.\egroup
  }{2017}]{liang:acl17}
Chen Liang, Jonathan Berant, Quoc Le, Kenneth Forbus, and Ni~Lao.
\newblock Neural symbolic machines: Learning semantic parsers on freebase with
  weak supervision.
\newblock In {\em ACL}, 2017.

\bibitem[\protect\citeauthoryear{Liu \bgroup \em et al.\egroup
  }{2016}]{liu:aaai16}
Shulin Liu, Kang Liu, Shizhu He, and Jun Zhao.
\newblock A probabilistic soft logic based approach to exploiting latent and
  global information in event classification.
\newblock In {\em AAAI}, 2016.

\bibitem[\protect\citeauthoryear{Manhaeve \bgroup \em et al.\egroup
  }{2018}]{manhaeve:neurips18}
Robin Manhaeve, Sebastijan Dumancic, Angelika Kimmig, Thomas Demeester, and Luc
  {De Raedt}.
\newblock {DeepProbLog}: Neural probabilistic logic programming.
\newblock In {\em NeurIPS}, 2018.

\bibitem[\protect\citeauthoryear{Manhaeve \bgroup \em et al.\egroup
  }{2021}]{manhaeve:ai21}
Robin Manhaeve, Sebastijan Duman{\v{c}}i{\'c}, Angelika Kimmig, Thomas
  Demeester, and Luc {De Raedt}.
\newblock Neural probabilistic logic programming in {DeepProbLog}.
\newblock {\em AI}, 298:103504, 2021.

\bibitem[\protect\citeauthoryear{Mao \bgroup \em et al.\egroup
  }{2019}]{mao:iclr19}
Jiayuan Mao, Chuang Gan, Pushmeet Kohli, Joshua~B Tenenbaum, and Jiajun Wu.
\newblock The neuro-symbolic concept learner: Interpreting scenes, words, and
  sentences from natural supervision.
\newblock In {\em ICLR}, 2019.

\bibitem[\protect\citeauthoryear{Mehta \bgroup \em et al.\egroup
  }{2018}]{mehta:emnlp18}
Sanket~Vaibhav Mehta, Jay~Yoon Lee, and Jaime Carbonell.
\newblock Towards semi-supervised learning for deep semantic role labeling.
\newblock In {\em EMNLP}, 2018.

\bibitem[\protect\citeauthoryear{Moldovan \bgroup \em et al.\egroup
  }{2015}]{bogdan:ecsqaru13}
Bogdan Moldovan, Ingo Thon, Jesse Davis, and Luc {de Raedt}.
\newblock Mcmc estimation of conditional probabilities in probabilistic
  programming languages.
\newblock In {\em ECSQARU}, 2015.

\bibitem[\protect\citeauthoryear{Nandwani \bgroup \em et al.\egroup
  }{2019}]{nandwani:neurips19}
Yatin Nandwani, Abhishek Pathak, and Parag Singla.
\newblock A primal dual formulation for deep learning with constraints.
\newblock In {\em NeurIPS}, 2019.

\bibitem[\protect\citeauthoryear{Pujara \bgroup \em et al.\egroup
  }{2013}]{pujara:iswc13}
Jay Pujara, Hui Miao, Lise Getoor, and William~W. Cohen.
\newblock Knowledge graph identification.
\newblock In {\em ISWC}, 2013.

\bibitem[\protect\citeauthoryear{Rockt{\"a}schel and
  Riedel}{2017}]{rocktaschel:neurips17}
Tim Rockt{\"a}schel and Sebastian Riedel.
\newblock End-to-end differentiable proving.
\newblock In {\em NeurIPS}, 2017.

\bibitem[\protect\citeauthoryear{Rospocher}{2018}]{rospocher:iswc18}
Marco Rospocher.
\newblock An ontology-driven probabilistic soft logic approach to improve nlp
  entity annotation.
\newblock In {\em ISWC}, 2018.

\bibitem[\protect\citeauthoryear{Sachan \bgroup \em et al.\egroup
  }{2018}]{sachan:neurips18}
Mrinmaya Sachan, Kumar~Avinava Dubey, Tom~M Mitchell, Dan Roth, and Eric~P
  Xing.
\newblock Learning pipelines with limited data and domain knowledge: A study in
  parsing physics problems.
\newblock In {\em NeurIPS}, 2018.

\bibitem[\protect\citeauthoryear{Sen \bgroup \em et al.\egroup
  }{2008}]{sen:aim08}
Prithviraj Sen, Galileo~Mark Namata, Mustafa Bilgic, Lise Getoor, Brian
  Gallagher, and Tina {Eliassi-Rad}.
\newblock Collective classification in network data.
\newblock {\em AI Magazine}, 29(3):93--106, 2008.

\bibitem[\protect\citeauthoryear{Serafini and {d'Avila
  Garcez}}{2016}]{serafini:aiia16}
Luciano Serafini and Artur~S. {d'Avila Garcez}.
\newblock Learning and reasoning with logic tensor networks.
\newblock In {\em AI*IA}, 2016.

\bibitem[\protect\citeauthoryear{Shalev-Shwartz}{2012}]{shalevshwartz:ftml11}
Shai Shalev-Shwartz.
\newblock Online learning and online convex optimization.
\newblock {\em Foundations and Trends in Machine Learning (FTML)},
  4(2):107--194, 2012.

\bibitem[\protect\citeauthoryear{Sikka \bgroup \em et al.\egroup
  }{2020}]{sikka:techreport20}
Karan Sikka, Andrew Silberfarb, John Byrnes, Indranil Sur, Ed~Chow, Ajay
  Divakaran, and Richard Rohwer.
\newblock Deep adaptive semantic logic (dasl): Compiling declarative knowledge
  into deep neural networks.
\newblock Technical report, SRI International, 2020.

\bibitem[\protect\citeauthoryear{Sridhar \bgroup \em et al.\egroup
  }{2018}]{sridhar:ijcai18}
Dhanya Sridhar, Jay Pujara, and Lise Getoor.
\newblock Scalable probabilistic causal structure discovery.
\newblock In {\em IJCAI}, 2018.

\bibitem[\protect\citeauthoryear{Srinivasan \bgroup \em et al.\egroup
  }{2021}]{srinivasan:mlj21}
Sriram Srinivasan, Charles Dickens, Eriq Augustine, Golnoosh Farnadi, and Lise
  Getoor.
\newblock A taxonomy of weight learning methods for statistical relational
  learning.
\newblock {\em Machine Learning}, 2021.

\bibitem[\protect\citeauthoryear{Wang \bgroup \em et al.\egroup
  }{2019}]{wang:icml19}
{Po-Wei} Wang, Priya Donti, Bryan Wilder, and Zico Kolter.
\newblock Satnet: Bridging deep learning and logical reasoning using a
  differentiable satisfiability solver.
\newblock In {\em ICML}, 2019.

\bibitem[\protect\citeauthoryear{Winters \bgroup \em et al.\egroup
  }{2021}]{winters:arxiv21}
Thomas Winters, Giuseppe Marra, Robin Manhaeve, and Luc {De Raedt}.
\newblock {DeepStochLog}: Neural stochastic logic programming.
\newblock {\em arXiv}, 2021.

\bibitem[\protect\citeauthoryear{Winters \bgroup \em et al.\egroup
  }{2022}]{winters:aaai22}
Thomas Winters, Giuseppe Marra, Robin Manhaeve, and Luc {De Raedt}.
\newblock {DeepStochLog}: Neural stochastic logic programming.
\newblock In {\em AAAI}, 2022.

\bibitem[\protect\citeauthoryear{Wu \bgroup \em et al.\egroup
  }{2019}]{wu:icml19}
Felix Wu, Amauri Souza, Tianyi Zhang, Christopher Fifty, Tao Yu, and Kilian
  Weinberger.
\newblock Simplifying graph convolutional networks.
\newblock In {\em ICML}, 2019.

\bibitem[\protect\citeauthoryear{Xu \bgroup \em et al.\egroup
  }{2018}]{xu:icml18}
Jingyi Xu, Zilu Zhang, Tal Friedman, Yitao Liang, and Guy~Van den Broeck.
\newblock A semantic loss function for deep learning with symbolic knowledge.
\newblock In {\em ICML}, 2018.

\bibitem[\protect\citeauthoryear{Yang \bgroup \em et al.\egroup
  }{2017}]{yang:neurips17}
Fan Yang, Zhilin Yang, and William~W. Cohen.
\newblock Differentiable learning of logical rules for knowledge base
  reasoning.
\newblock In {\em NeurIPS}, 2017.

\bibitem[\protect\citeauthoryear{Yang \bgroup \em et al.\egroup
  }{2020}]{yang:ijcai20}
Zhun Yang, Adam Ishay, and Joohyung Lee.
\newblock Neurasp: Embracing neural networks into answer set programming.
\newblock In {\em IJCAI}, 2020.

\end{thebibliography}

    \begin{appendix}
        \newpage{}
\section{Appendix}

The appendix includes the following sections: Limitations, Formulating Existing NeSy Frameworks as NeSy-EBMs, Joint Reasoning in NeSy-EBMs, \shortname{} Parameter Learning, Dataset Details, \shortname \ Models, Baseline Models, Extended Evaluation Details, and Computational Hardware Details.
        \section{Limitations}

Practitioners applying NeuPSL should consider the following three limitations.
First, NeuPSL operates on real-valued logic, which improves scalability but is a relaxation of the original problem.
This relaxation may overlook nuances (e.g., integer constraints) of the original task.
Second, while NeuPSL demonstrates excellent performance in solving joint symbolic inference tasks, it comes at the expense of a higher inference runtime than a purely neural model.
The computational demands of NeuPSL may limit its applicability in scenarios where real-time processing is necessary.
Lastly, NeuPSL is trained in this work with the energy learning loss.
Using this loss reduces the energy of the truth data but does not necessarily align with a downstream evaluation metric, and we have identified some degenerate solutions (\appref{sec:energy_degenerate_solutions}).
Exploring the adaptation of NeuPSL to support different learning losses is an interesting avenue for future research.

        \section{Formulating Existing NeSy Frameworks as NeSy-EBMs}
\label{sec:appendix_other_NeSy-EBMS}

This section shows how to formulate two popular NeSy frameworks, \dpllongname{} (\dplshortname{}) \citep{manhaeve:neurips18} and \ltnshortname{} (\ltnshortname{}) \citep{badreddine:ai22}, as NeSy-EBMs.

\subsection{\dpllongname{}} 
\dpllongname{} (\dplshortname{}) extends the probabilistic programming language ProbLog \citep{deraedt:ijcai07}.
A ProbLog program consists of (i) a set of probabilistic facts $\mathcal{F}$ of the form $p\, :: \, f$ where $p$ is a probability and $f$ is a $\{0, 1\}$ valued symbolic variable and (ii) a set $\mathcal{R}$ of symbolic statements or rules.
The following ProbLog program is a common example that models the likelihood of a burglary or an earthquake, given an alarm was sounded and is also presented in \cite{manhaeve:ai21}
\begin{align*}
    &\textrm{\# Probabilistic Facts}\\
    & 0.1 \, :: \, \textrm{burglary}. \quad 0.2 \, :: \, \textrm{earthquake}.\\
    & 0.5 \, :: \, \textrm{hearsAlarm(mary)}. \quad 0.4 \, :: \, \textrm{hearsAlarm(john)}. \\
    &\textrm{\# Rules}\\
    & \textrm{alarm} \, :- \, \textrm{earthquake}. \\
    & \textrm{alarm} \, :- \, \textrm{burglary}. \\
    & \textrm{calls(X)} \, :- \, \textrm{alarm, hearsAlarm(X)}. \\
\end{align*}
A subset of the probabilistic facts $F \subseteq \mathcal{F}$ defines a possible instantiation, or world:
\begin{align*}
    t_{F} := F \cup \{ f \, \vert \, \mathcal{R} \cup F \, \vDash f \}.
\end{align*}
For the example, 
$
    t_{\{\textrm{burglary, hearsAlarm(mary)}\}} = \{ \textrm{burglary, hearsAlarm(mary), alarm, calls(mary)} \}
$.
Then, the probability of a world, $P(t_{F})$, is the product of the probabilities of the probabilistic facts in the world:
\begin{align*}
    P(t_{F}) := \Pi_{\mathbf{f}[i] \in F} \mathbf{p}[i] \Pi_{\mathbf{f}[i] \in \mathcal{F} \setminus F}(1 - \mathbf{p}[i]).
\end{align*}
For the running example, 
$
    P(t_{\{\textrm{burglary, hearsAlarm(mary)}\}}) = 0.1 \cdot 0.5 \cdot (1- 0.2) \cdot (1 - 0.4)
$.
Finally, the probability of a query atom, $q$, is defined as the sum of the probabilities of the worlds containing $q$:
\begin{align*}
    P(q) := \sum_{F \in \mathcal{F} \, \vert \, q \in t_{F}} P(t_{F}).
\end{align*}

ProbLog inference, specifically as it is applied in the deep extension proposed by \citenoun{manhaeve:neurips18}, is a marginal inference problem.
Specifically, the inference task is computing the marginal probability of a single query atom as shown above.
This is equivalent to finding the weighted model count (WMC) of the worlds where the query atom is true.
Thus, the exact marginal inference problem in ProbLog is \#P-complete, i.e., it is at least NP-hard. 
This means that computing the exact probability of a query in a ProbLog program is a computationally challenging problem that requires exponential time in the worst case.
Therefore, exact marginal inference in ProbLog is generally only feasible for small or moderately sized problems. 
For larger problems with more variables, approximate inference techniques are used to obtain approximate probabilities more efficiently \cite{deraedt:ijcai07, bogdan:ecsqaru13}.

\dplshortname{} introduces syntax and semantics to ProbLog to support specifying probabilities of events with neural networks \cite{manhaeve:neurips18, manhaeve:ai21}.
Specifically, a set of neural annotated disjunctions (nADs) are specified by a user and take the form:
\begin{align*}
     &\textrm{nn}(id, \mathbf{v}, u_{1}) \, :: \, \textrm{h}(\mathbf{v}, u_{1}) \, ; \,\\
     & \quad \cdots \, ; \,
     \textrm{nn}(id, \mathbf{v}, u_{n}) \, :: \, \textrm{h}(\mathbf{v}, u_{n}) \, ; \,
     \vDash b_{1}, \cdots, b_{m},
\end{align*}
where the $b_{i}$ are atoms, $\mathbf{v}$ is a vector of features that the neural component, identified by $id$, has access to.
Moreover, the output of the neural component, $\textrm{nn}(id, \mathbf{v}, u_{i})$, is interpreted as the probability that the atom $h_{i}$ is true and the sum of the outputs of the neural model must sum to 1:
$\sum_{i = 1}^{n} \textrm{nn}(id, \mathbf{v}, u_{1}) = 1$.
The interpretation of an annotated disjunction is that whenever all of the atoms $b_{1}, \cdots, b_{m}$ are true, then each $h_{i}$ will be true with probability $\textrm{nn}(id, \mathbf{v}, u_{i})$.

Inference in \dplshortname{} is exactly the same as ProbLog marginal inference with a single query atom, except a forward pass is made with the neural network to compute the probabilities of the nADs.
Learning the parameters of the \dplshortname{} model is the task of finding the setting of the trainable parameters, denoted by $\mathbf{x}$, that minimizes a sum of losses, $L()$.
Each loss measures the distance between a vector of $n$ desired probabilities $\mathbf{p}_{true}$ and $[P(q_{1}), \cdots, P(q_{n})]$, the marginal inference values predicted by \dplshortname{}:
\begin{align*}
    \argmin_{\mathbf{x}} \frac{1}{n} \sum_{i = 1}^{n} L(P(q_{i}), \mathbf{p}_{true}[i]).
\end{align*}

Though instantiating the marginal probability function is non-trivial and computationally expensive, marginal inference ultimately reduces to a series of differentiable algebraic operations and is therefore differentiable.
\dplshortname{} uses stochastic gradient descent to find parameters minimizing the training objective.

\dplshortname{} is a NeSy-EBM.
The fact probabilities, $p$, are partitioned into the observed NeSy-EBM symbolic variables $\mathbf{x}_{sy}$, the vector of symbolic parameters, $\mathbf{w}_{sy}$, and neural network outputs, $\mathbf{g}(\mathbf{x}_{nn}, \mathbf{w}_{nn})$.
Without loss of generality, suppose 
\begin{align*}
    \mathbf{p} = \begin{bmatrix}
        \mathbf{x}_{sy} \\
        \mathbf{w}_{sy} \\
        \mathbf{g}(\mathbf{x}_{nn}, \mathbf{w}_{nn})
    \end{bmatrix}.
\end{align*}
The query atoms, i.e., the atoms present in the \dplshortname{} model that are not specified in the set of probabilistic facts, correspond to the symbolic variables $\mathbf{y}$.

The definition of the \dplshortname{} symbolic potentials and energy function are tied to the inference task; a different definition of the symbolic potential and energy function is used to implement marginal versus MAP inference.
As previously mentioned, \dplshortname{} predictions are most commonly obtained by performing marginal inference for a single query atom.
Moreover, a consequence of the \dplshortname{} semantics is that the marginal inference problem reduces to an analytical expression composed of only product and sum operations.
Thus, from the NeSy EBM perspective, to implement marginal inference \dplshortname{} interprets a program with a set of probabilistic facts and data to define a symbolic potential for every marginal probability and then the energy function is simply the sum of the symbolic potentials.
On the other hand, for MAP inference, \dplshortname{} creates a symbolic potential for every possible world and the energy function is equivalent to the negative of the joint probability distribution implied by the \dplshortname{} program.
We will only formally cover the marginal inference case.

The probability of a world $t_{F}$, defined by the subset of probabilistic facts $F \in \mathcal{F}$ is a function of the \dplshortname{} fact probabilities, $\mathbf{p}$, and hence is a function of $\mathbf{x}_{sy}$, $\mathbf{w}_{sy}$, and $\mathbf{g}(\mathbf{x}_{nn}, \mathbf{w}_{nn})$:
\begin{align*}
    & P_{\mathbf{t}_{F}}(\mathbf{x}_{sy}, \mathbf{w}_{sy}, \mathbf{g}(\mathbf{x}_{nn}, \mathbf{w}_{nn})) \\
    & \quad := \bigg( \Pi_{\mathbf{x}_{sy}[j] \in F} \mathbf{x}_{sy}[j]  \bigg) \cdot  \bigg( \Pi_{\mathbf{x}_{sy}[j] \in \mathcal{F} \setminus F} (1 - \mathbf{x}_{sy}[j]) \bigg ) \\
    & \quad \quad \cdot \bigg( \Pi_{\mathbf{w}_{sy}[j] \in F} \mathbf{w}_{sy}[j]  \bigg) \cdot  \bigg( \Pi_{\mathbf{w}_{sy}[j] \in \mathcal{F} \setminus F} (1 - \mathbf{w}_{sy}[j]) \bigg )\\
    & \quad \quad \cdot \bigg( \Pi_{\mathbf{g}(\mathbf{x}_{nn}, \mathbf{w}_{nn})[j] \in F} \mathbf{g}(\mathbf{x}_{nn}, \mathbf{w}_{nn})[j]  \bigg) \\
    & \quad \quad \cdot \bigg( \Pi_{\mathbf{g}(\mathbf{x}_{nn}, \mathbf{w}_{nn})[j] \in \mathcal{F} \setminus F} (1 - \mathbf{g}(\mathbf{x}_{nn}, \mathbf{w}_{nn})[j]) \bigg).
\end{align*}
Then, as in ProbLog, the marginal probability of a query atom is a function of the probabilities of the worlds.
For the world $t_{F}$, defined by the subset of probabilistic facts $F \subseteq \mathcal{F}$, let $\chi_{t_{F}}[\cdot]$ be the indicator function identifying if a setting of the variables $\mathbf{y}$ matches the world $t_{F}$:
{\small
\begin{align*}
    \chi_{t_{F}}[\hat{\mathbf{y}}] := 
    \begin{cases} 
    1 & ((\hat{\mathbf{y}}[i] = 1) \implies \mathbf{y}[i] \in t_{F}) \, \forall i \in \{1, \cdots, n_{\mathbf{y}} \} \\
    0 & \textrm{o.w.}
    \end{cases}.
\end{align*}
}%
With $\chi_{t_{F}}[\mathbf{y}]$, it is also possible to write the marginal probability of a variable as function of $\mathbf{x}_{sy}, \mathbf{w}_{sy}$, and $\mathbf{g}(\mathbf{x}_{nn}, \mathbf{w}_{nn})$:
{\small
\begin{align*}
    & P_{\mathbf{y}[i]}(\mathbf{x}_{sy}, \mathbf{w}_{sy}, \mathbf{g}(\mathbf{x}_{nn}, \mathbf{w}_{nn})) \\
    & \, := \sum_{\hat{\mathbf{y}} \in \{0, 1\}^{n_{\mathbf{y}}}} \hat{\mathbf{y}}[i] \left (\sum_{F \in \mathcal{P}(\mathcal{F})} \chi_{t_{F}}[\hat{\mathbf{y}}] P_{\mathbf{t}_{F}}(\mathbf{x}_{sy}, \mathbf{w}_{sy}, \mathbf{g}(\mathbf{x}_{nn}, \mathbf{w}_{nn})) \right).
\end{align*}
}%
Let $d: [0, 1] \times [0, 1] \to \mathbb{R}$ be a metric quantifying the distance between its two arguments.
For each variable $\mathbf{y}[i]$ for $i \in \{1, \cdots, n_{\mathbf{y}}\}$ define a symbolic potential:
{\small
\begin{align*}
    &\psi_{\dplshortname{}, i}(\mathbf{y}, \mathbf{x}_{sy}, \mathbf{w}_{sy}, \mathbf{g}(\mathbf{x}_{nn}, \mathbf{w}_{nn})) \\
    & \, := d(\mathbf{y}[i], P_{\mathbf{y}[i]}(\mathbf{x}_{sy}, \mathbf{w}_{sy}, \mathbf{g}(\mathbf{x}_{nn}, \mathbf{w}_{nn}))).
\end{align*}
}%
Let $\Psi_{\dplshortname{}}(\cdot) := \left [ \psi_{\dplshortname{}, t_{F_{i}}}(\cdot) \right ]_{i = 1}^{n_{\mathbf{y}}}$ be the vector of all $n_{\mathbf{y}}$ symbolic potentials.
The energy function to produce marginal inference \dplshortname{} predictions is then the summation of all the symbolic potentials:
{\small
\begin{align*}
    & E_{\dplshortname{}}(\Psi_{\dplshortname{}}(\mathbf{y}, \mathbf{x}_{sy}, \mathbf{w}_{sy}, \mathbf{g}(\mathbf{x}_{nn}, \mathbf{w}_{nn}))) \\
    & \, := \sum_{i = 1}^{n_{\mathbf{y}}} \Psi_{\dplshortname{}}(\mathbf{y}, \mathbf{x}_{sy}, \mathbf{w}_{sy}, \mathbf{g}(\mathbf{x}_{nn}, \mathbf{w}_{nn}))[i].
\end{align*}
}%
Clearly, the optimal value of the energy function is $0$ and is achieved at the unique setting of the variables matching their corresponding marginal probability.
Thus inference is equivalent to evaluating the marginal probabilities for each variable.

\subsection{\ltnlongname{}}

\ltnlongname{}s (\ltnshortname{}) forwards deep neural network predictions into functions representing symbolic relations with real-valued or fuzzy logic semantics \cite{badreddine:ai22}.
The fuzzy logic functions are combined using a \emph{formula aggregator} to define a satisfaction level.
\citenoun{badreddine:ai22} suggest using the product real logical semantics to translate logical statements, i.e., given two truth values $a$ and $b$ in $[0, 1]$:
\begin{align*}
    \neg(a) &:= 1 - a \nonumber \\
    \land(a, b) &:= a \cdot b \nonumber \\
    \vee(a, b) &:= a + b - a \cdot b \nonumber \\
    \implies(a, b) &:= a + b - a \cdot b 
\end{align*}
Additionally, generalized mean semantics for existential and universal quantifiers are used for collections of truth values $\mathbf{a} = [a]_{i = 1}^{n}$:
\begin{align*}
    \exists(\mathbf{a}) &:= \left( \frac{1}{n} \sum_{i = 1}^{n} a_{i}^p \right)^{\frac{1}{p}} \\
    \forall(\mathbf{a}) &:= 1 - \left( \frac{1}{n} \sum_{i = 1}^{n} (1 - a_{i})^p \right)^{\frac{1}{p}},
    \label{eq:real-squantifer}
\end{align*}
where $p \in \mathbb{R}_{+}$ is a hyperparameter.
For example, consider the logical statement 
\begin{align*}
    \exists v \in \mathcal{V} \left( P(u, v) \land Q(v)\right).
\end{align*} 
\ltnshortname{} instantiate predicate arguments with features.
Let $\mathbf{X}_{\mathcal{U}}$ and $\mathbf{X}_{\mathcal{V}}$ be collections of variable feature vectors such that $\mathbf{X}[u]$ and $\mathbf{X}[v]$ are the feature vectors corresponding to the entities $u$ and $v$, respectively.
Furthermore, the predicate values are either provided by a deep neural network output or are values representing observations or a potential prediction.
For instance, the predicate $P(u, v)$ in the example can be instantiated as the output of a deep neural network parameterized by its weights $\mathbf{w}_{nn}$ and represented by the function: $nn(\mathbf{X}[u], \mathbf{X}[v]; \mathbf{w})$ which takes the two feature vectors, corresponding to the arguments $u$ and $v$, respectively, to a value in $[0, 1]$.
Then, $Q(v)$ could be a constant from $[0, 1]$.
Let $\mathbf{x}_{\mathcal{Q}}$ be a vector of scalars from $[0, 1]$ such that $\mathbf{x}_{\mathcal{Q}}[v]$ represents the predicate value for $Q(v)$.
Then, the logical statement in the example is a composition of the specified real-logic operators and quantifiers.
For a provided instance of the argument $u$ the real-valued logic function for the example is:
{\small
\begin{align*}
    & h_{u}(\mathbf{X}_{\mathcal{U}}, \mathbf{X}_{\mathcal{V}}, \mathbf{x}_{\mathcal{Q}}; \mathbf{w}) \\
    & := \bigg( \frac{1}{\Vert \mathcal{V} \Vert} \sum_{v \in \mathcal{V}} \big( nn(\mathbf{X}[u], \mathbf{X}[v]; \mathbf{w}) \cdot \mathbf{x}_{Q}[v] \big)^p \bigg)^{\frac{1}{p}}.
\end{align*}
}%
Using the generalized mean semantics for the universal quantifier as the formula aggregator, the satisfaction level of the \ltnshortname{} model prediction is:
{\small
\begin{align*}
    G(\mathbf{w}) := 1 - \bigg(\frac{1}{\Vert \mathcal{U} \Vert} \sum_{u \in \mathcal{U}} \big(1 - h_{u}(\mathbf{X}_{\mathcal{U}}, \mathbf{X}_{\mathcal{V}}, \mathbf{x}_{\mathcal{Q}}; \mathbf{w}) \big)^p \bigg)^{\frac{1}{p}}.
\end{align*}
}%
There are many ways to instantiate an LTN depending on the modeler's choice of real-logic semantics, the formula aggregator, and the logical relations.
The example above illustrates a common setting of the real-logic semantics and the formula aggregator for a specific composition of logical formula.

The parameters of the \ltnshortname{} are the deep neural network weights.
Learning is the task of finding a setting of the weights which maximize the satisfaction of an aggregated set of logical formula instantiated with observations and features:
\begin{align*}
    \mathbf{w}^{*} = \argmax_{\mathbf{w}} G(\mathbf{w}) .
\end{align*}
In other words, learning in LTNs can be understood as optimizing under
first-order logic constraints relaxed into a loss function.
There are a variety of real-valued logical semantics and formula aggregators that result in the satisfaction level function $G(\mathbf{w})$ being differentiable with respect to the weights.
Given a trained set of parameters obtained by learning, $\mathbf{w}^{*}$, inference is presented as querying the truth value of an instantiated predicate or logical formula.
A prediction in \ltnshortname{} in a multi-class or joint output setting such is obtained by evaluating the truth-value of all possible outputs and returning the highest valued configuration, i.e., the state with maximum satisfaction.

Through the lens of NeSy-EBMs, the system's fuzzy logic semantics define the symbolic potentials and the formula aggregator is the energy function.
More formally, the NeSy-EBM unobserved and observed symbolic variables and neural network outputs partition the instantiated predicates of the real-valued logic functions $h_{i}$.
Each of the $m$ real-valued logic functions can be written as a function of only the symbolic variables and the neural network outputs: $h_{i}(\mathbf{y}, \mathbf{x}_{sy}, \mathbf{g}(\mathbf{x}_{nn}, \mathbf{w}_{nn}))$.
The functions $h_{i}$ are the symbolic potentials of the NeSy-EBM:
{\small
\begin{align*}
    \psi_{LTN,i}(\mathbf{y}, \mathbf{x}_{sy}, \mathbf{w}_{sy}, \mathbf{g}(\mathbf{x}_{nn}, \mathbf{w}_{nn}) := h_{i}(\mathbf{y}, \mathbf{x}_{sy}, \mathbf{g}(\mathbf{x}_{nn}, \mathbf{w}_{nn}).
\end{align*}
}%
Let $\Psi_{\ltnshortname{}}(\cdot) := \left [ \psi_{\ltnshortname{}, i}(\cdot) \right]_{i = 1}^{m}$ be the vector of all $m$ symbolic potentials.
Then, the formula aggregator defines the energy function.
Using the generalized mean semantics for the universal quantifier, the NeSy-EBM energy function for \ltnshortname{} is:
{\small
\begin{align*}
    & E_{LTN}(\Psi_{\ltnshortname{}}(\mathbf{y}, \mathbf{x}_{sy}, \mathbf{w}_{sy}, \mathbf{g}(\mathbf{x}_{nn}, \mathbf{w}_{nn}))) \\
    & := \bigg(\frac{1}{m} \sum_{i = 1}^{m} \big( 1 - \Psi_{\ltnshortname{}}(\mathbf{y}, \mathbf{x}_{sy}, \mathbf{w}_{sy}, \mathbf{g}(\mathbf{x}_{nn}, \mathbf{w}_{nn}))[i] \big)^{p} \bigg)^{\frac{1}{p}}.
\end{align*}
}%
The \ltnshortname{} framework is general and the scalability and expressivity of the system are dependent on the modeler's choice of the domain of the unobserved variables: $\mathcal{Y}$, the real-valued logical semantics, and the formula aggregator.
Furthermore, notice there is no explicit use of the symbolic parameters $\mathbf{w}_{sy}$ as the \ltnshortname{} framework uses standard real-valued logics that typically do not have trainable parameters.

\ltnshortname{} learning is finding the parameters with the highest satisfaction, i.e., learning with the energy loss in the NeSy-EBM framework.
The NeSy-EBM framework connects \ltnshortname{} to the EBM literature, which suggests principled alternative learning algorithms.
Moreover, the NeSy-EBM framework sheds light on design choices for the various components of the \ltnshortname{} to ensure the applicability of first-order methods for learning and desirable scalability and expressiveness properties of inference.

        \section{Joint Reasoning in  NeSy-EBMs}
\label{sec:appendix_joint_reasoning}

This section expands the discussion of joint reasoning in NeSy-EBMs.
To reiterate, we highlight two important categories of NeSy-EBM energy functions: \emph{joint} and \emph{independent}.
Formally, an energy function that is additively separable over the output variables $\mathbf{y}$ is an \textit{independent energy function}, i.e., corresponding to each of the $n_{y}$ components of the output variable $\mathbf{y}$ there exists functions $n_{y}$ functions $E_{1}(\mathbf{y}[1], \mathbf{x}_{sy}, \mathbf{w}_{sy}, \mathbf{g}(\mathbf{x}_{nn}, \mathbf{w}_{nn}))$, $\cdots$, $E_{n_{y}}(\mathbf{y}[n_{y}], \mathbf{x}_{sy}, \mathbf{w}_{sy}, \mathbf{g}(\mathbf{x}_{nn}, \mathbf{w}_{nn}))$ such that 
\begin{align*}
    E(\cdot) = \sum_{i = 1}^{n_{y}} E_{i}(\mathbf{y}[i], \mathbf{x}_{sy}, \mathbf{w}_{sy}, \mathbf{g}(\mathbf{x}_{nn}, \mathbf{w}_{nn})).
\end{align*}
While a function that is not separable over output variables $\mathbf{y}$ is a \textit{joint energy function}.
This categorization allows for an important distinction during inference and learning.
Independent energy functions simplify inference and learning as finding an energy minimizer, $\mathbf{y}^{*}$, can be distributed across the independent functions $E_{i}$.
In other words, the predicted value for a variable $\mathbf{y}[i]$ has no influence over that of $\mathbf{y}[j]$ where $j \neq i$ and can therefore be predicted separately, i.e., independently.
However, independent energy functions cannot leverage some joint information that may be used to improve predictions.

To illustrate, recall the example described in the \textit{Neural Probabilistic Soft Logic} section where a neural network is used to classify the species of an animal in an image with external information.
\figref{fig:joint-non-joint-energy-function} outlines the distinction between independent and joint prediction for this scenario.
In \figref{fig:joint-non-joint-energy-function}(a), the independent setting, the input is a single image, and the energy function is defined over the three possible classes: \pslarg{dog}, \pslarg{cat}, and \pslarg{frog}.
While in \figref{fig:joint-non-joint-energy-function}(b), the joint setting, the input is a pair of images, and the energy function is defined for every possible combination of labels (e.g., (\pslarg{dog}, \pslarg{dog}), (\pslarg{dog}, \pslarg{cat}), etc.).
The joint energy function of (b) leverages external information suggesting the images are of the same entity.
Joint reasoning enables a model to make structured predictions that resolve contradictions an independent model could not detect.

For NeSy-EBMs, a joint energy function encodes dependencies between its output variables through its symbolic potentials.
\shortname{} additionally benefits from scalable convex inference to speed up learning over a dependent set of output variables.
As we see in the \textit{Experimental Evaluation} section, utilizing joint inference and learning in NeSy-EBMs not only provides a boost in performance but produces results that non-joint methods cannot (even with five times the amount of data).

\begin{figure}[t]
    \centering
    \begin{subfigure}[b]{0.50\textwidth}
        \centering
        \includegraphics[width=\textwidth]{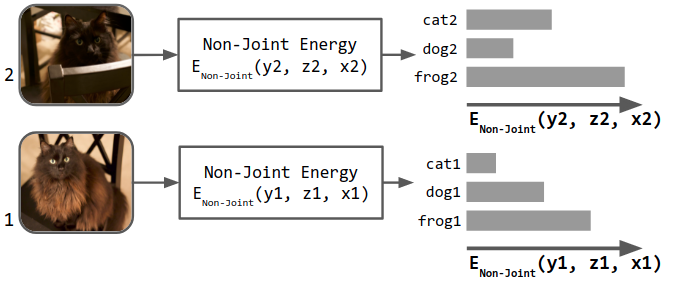}
        \caption{Independent Energy}
    \end{subfigure}
    \begin{subfigure}[b]{0.50\textwidth}
        \centering
        \includegraphics[width=\textwidth]{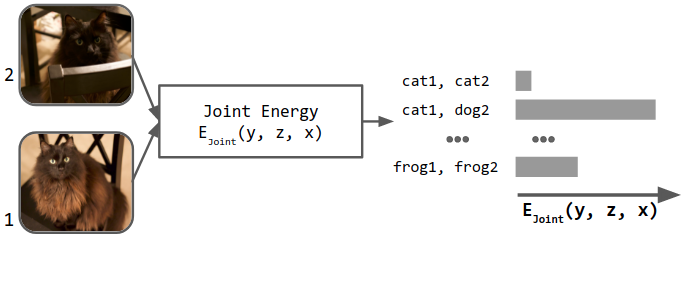}
        \caption{Joint-Energy}
    \end{subfigure}
    \caption{
        Example of non-joint and joint energy functions.
    }
    \label{fig:joint-non-joint-energy-function}
\end{figure}
        \section{\shortname{} Parameter Learning}
\label{sec:appendix_neupsl}

This section details the \shortname{} parameter learning algorithm.
We begin by discussing degenerate solutions to the energy loss problem and techniques for overcoming them.
We then provide the precise parameter updates we use to efficiently fit \shortname{} model parameters while avoiding the discussed degenerate solutions.

\subsection{Energy Loss Degenerate Solutions}
\label{sec:energy_degenerate_solutions}

In this section, we show two degenerate solutions of the energy loss learning problem for \shortname{} and methods for overcoming them.
Recall that the \shortname{} energy loss learning problem is:
{\footnotesize
\begin{align*}
    & \argmin_{(\mathbf{w}_{nn}, \mathbf{w}_{psl}) \in \mathbb{R}^{nn} \times \mathbb{R}^{r}_{+}} \mathcal{L}(\mathbf{w}_{nn}, \mathbf{w}_{psl}, \mathcal{S}) \\
    & = \argmin_{(\mathbf{w}_{nn}, \mathbf{w}_{psl}) \in \mathbb{R}^{nn} \times \mathbb{R}^{r}_{+}} \sum_{i = 1}^{P} E((\mathbf{y}_{i, t}, \mathbf{z}^{*}_{i}), \mathbf{x}_{sy, i}, \mathbf{x}_{nn, i}, \mathbf{w}_{nn}, \mathbf{w}_{psl}) \\
    & = \argmin_{(\mathbf{w}_{nn}, \mathbf{w}_{psl}) \in \mathbb{R}^{nn} \times \mathbb{R}^{r}_{+}} \\
    & \sum_{i = 1}^{P} \min_{\mathbf{z} \vert ((\mathbf{y}_{i, t}, \mathbf{z}), \mathbf{x}) \in \mathbf{\Omega}} \mathbf{w}_{psl}^T \mathbf{\Phi}((\mathbf{y}_{i, t}, \mathbf{z}), \mathbf{x}_{sy, i}, \mathbf{x}_{nn, i}, \mathbf{w}_{nn})
\end{align*}
}%
Note that the symbolic parameters are constrained to be non-negative real numbers.
Furthermore, as every symbolic potential has the form:
{\small
\begin{align*}
    \phi_{i}(\mathbf{y}, \mathbf{x}_{sy}, \mathbf{x}_{nn}, \mathbf{w}_{nn}) = \max(l_{i}(\mathbf{y}, \mathbf{x}_{sy}, \mathbf{g}_{nn}(\mathbf{x}_{nn}, \mathbf{w}_{nn})), 0)^\alpha
\end{align*}
}%
we have that $\phi_{i}(\mathbf{y}, \mathbf{x}_{sy}, \mathbf{x}_{nn}, \mathbf{w}_{nn}) \geq 0$ for all settings of the variables $\mathbf{y}, \mathbf{x}_{sy}, \mathbf{x}_{nn}, \mathbf{w}_{nn}$.
Thus, $\Phi_i(\mathbf{y} , \mathbf{x}_{sy}, \mathbf{x}_{nn}, \mathbf{w}_{nn}) := \sum_{j \in t_i} \phi_{i}(\mathbf{y}, \mathbf{x}_{sy}, \mathbf{x}_{nn}, \mathbf{w}_{nn}) \geq 0$ and  $\mathbf{\Phi(\mathbf{y} , \mathbf{x}_{sy}, \mathbf{x}_{nn}, \mathbf{w}_{nn})} := [\Phi_i(\mathbf{y} , \mathbf{x}_{sy}, \mathbf{x}_{nn}, \mathbf{w}_{nn})]_{i = 1}^{r} \succeq \mathbf{0}$.
Therefore, we have that 
\begin{align*}
    & \mathcal{L}(\mathbf{w}_{nn}, \mathbf{w}_{psl}, \mathcal{S}) = \\
    & \sum_{i = 1}^{P} \min_{\mathbf{z} \vert ((\mathbf{y}_{i, t}, \mathbf{z}), \mathbf{x}_{sy}) \in \mathbf{\Omega}} \mathbf{w}_{psl}^T \mathbf{\Phi}((\mathbf{y}_{i, t}, \mathbf{z}), \mathbf{x}_{sy, i}, \mathbf{x}_{nn, i}, \mathbf{w}_{nn}) \geq 0 
\end{align*}
In fact, $\mathcal{L}(\mathbf{w}_{nn}, \mathbf{w}_{psl}, \mathcal{S}) = 0$ when $\mathbf{w}_{psl} = \mathbf{0}$.
The $\mathbf{0}$ solution to the weight learning problem is degenerate and should be avoided.
Precisely, $\mathbf{w}_{psl} = \mathbf{0}$ results in a \emph{collapsed energy function}: a function that assigns all points $\mathbf{y} \in \mathcal{Y}$ to the same energy.
Collapsed energy functions have no predictive power since inference, i.e., finding a lowest energy state of the variables is trivial and uninformative.
To overcome this degenerate solution a simplex constraint on the symbolic parameters, $\mathbf{w}_{psl} \in \Delta^r := \{\mathbf{w} \in \mathbb{R}^r_+ \big \vert \Vert \mathbf{w} \Vert_{1} = 1\}$, is added, making the degenerate solution $\mathbf{w}_{psl} = \mathbf{0}$ infeasible.
This constraint also ensures the non-negativity of the parameters and does not inhibit the expressivity of \shortname{} when the deep HL-MRF is exclusively used to obtain MAP inference predictions. 
This property of (deep) HL-MRFs was shown by \citenoun{srinivasan:mlj21}, where they proved and leveraged the fact that MAP inference in HL-MRFs is invariant to the scale of the weights.
Formally, for all weight configurations $\mathbf{w}_{psl}$ and scalars $\tilde{c} \in \mathbb{R}_{+}$, 
\begin{align*}
    \argmax_{\mathbf{y} \vert (\mathbf{y}, \mathbf{x}_{sy}) \in \mathbf{\Omega}} & E(\mathbf{y}, \mathbf{x}_{sy}, \mathbf{x}_{nn}, \mathbf{w}_{nn}, \mathbf{w}_{psl}) \nonumber \\
    & = \argmax_{\mathbf{y} \vert (\mathbf{y}, \mathbf{x}_{sy}) \in \mathbf{\Omega}} E(\mathbf{y}, \mathbf{x}_{sy}, \mathbf{x}_{nn}, \mathbf{w}_{nn}, \tilde{c} \cdot \mathbf{w}_{psl})
\end{align*}

The $\mathbf{w}_{psl} = \mathbf{0}$ is infeasible with the simplex constraint; however, an additional degenerate solution arises from its introduction.
This is because the energy loss is concave in the symbolic parameters $\mathbf{w}_{psl}$ for fixed $\mathbf{w}_{nn}$ and $\mathcal{S}$, as is shown in following lemma and its corresponding proof.
Consequently, a solution to the constrained energy loss learning problem must exist at corner points of the simplex.

\begin{lemma}
    The energy loss function
    \begin{align*}
        \mathcal{L}(\mathbf{w}_{nn}, &  \mathbf{w}_{psl}, \mathcal{S}) = \\
        & \sum_{i = 1}^{P} \min_{\mathbf{z} \vert ((\mathbf{y}_{i, t}, \mathbf{z}), \mathbf{x}) \in \mathbf{\Omega}} \mathbf{w}_{psl}^T \mathbf{\Phi}((\mathbf{y}_{i, t}, \mathbf{z}), \mathbf{x}_{sy, i}, \mathbf{x}_{nn, i}, \mathbf{w}_{nn})
    \end{align*}
    is concave in $\mathbf{w}_{psl}$.
\end{lemma}

\begin{proof}
    For all $i$
    \begin{align*}
        E((\mathbf{y}_{i, t}, & \mathbf{z}^{*}_{i}), \mathbf{x}_{sy, i}, \mathbf{x}_{nn, i}, \mathbf{w}_{nn}, \mathbf{w}_{psl}) = \\
        & \inf_{\mathbf{z} \vert ((\mathbf{y}_{i, t}, \mathbf{z}), \mathbf{x}_{sy}) \in \mathbf{\Omega}} \mathbf{w}_{psl}^T \Phi((\mathbf{y}_{i, t}, \mathbf{z}), \mathbf{x}_{sy, i}, \mathbf{x}_{nn, i}, \mathbf{w}_{nn}) 
    \end{align*}
    is a pointwise infimum of a set of affine, hence concave, functions of $\mathbf{w}_{psl}$ and is therefore concave \citep{boyd:book04}.
    Therefore, 
    \begin{align}
        \mathcal{L}(\mathbf{w}_{nn}, \mathbf{w}_{psl}, \mathcal{S}) 
        = &\sum_{i = 1}^{P} E((\mathbf{y}_{i, t}, \mathbf{z}^{*}_{i}), \mathbf{x}_{sy, i}, \mathbf{x}_{nn, i}, \mathbf{w}_{nn}, \mathbf{w}_{psl})
    \end{align}
    is a sum of concave functions of $\mathbf{w}_{psl}$ and is concave.
\end{proof}

Additionally, note that the unit simplex, $\Delta^r$, is a convex set, and, more precisely, a polyhedron.
Following from its definition, a concave function is minimized over a polyhedron at one of the vertices.
This solution is undesirable for the energy minimization problem because each symbolic relation corresponding to the parameters should have an influence over the model predictions.
For this reason, we propose using a negative logarithm as a parameter regularizer, giving the simplex corner solutions infinitely high energy.
With negative log regularization and simplex constraints, energy loss symbolic parameter learning is:
\begin{align}
    & \min_{\mathbf{w}_{nn} \in \mathcal{W}_{nn}, \mathbf{w}_{psl} \in \Delta^r}  \mathcal{L}(\mathbf{w}_{nn}, \mathbf{w}_{psl}, \mathcal{S}) - \sum_{i=1}^{r} \log_{b}(\mathbf{w}_{psl}[i])
\end{align}

\subsection{Exponentiated Gradient Descent}
\label{sec:exponentiated_gradient}

As suggested by \citenoun{dickens:tpm22}, we minimize the energy loss with respect to the symbolic parameters constrained to the unit simplex via normalized exponentiated gradient descent \citep{kivinen:ic97, shalevshwartz:ftml11}.
Then, minimization over neural parameters is performed with standard gradient descent.
With an initial step size parameter $\eta > 0$, the parameter updates are
{\small
\begin{align*}
    \mathbf{w}_{nn}^{k + 1} &= \mathbf{w}_{nn}^{k} + \eta \nabla_{w_{nn}} \mathcal{L}(\mathbf{w}_{nn}^{k}, \mathbf{w}_{psl}^{k}, S) \\
    \mathbf{w}_{psl}^{k + 1}[i] &= \frac{\mathbf{w}_{psl}^{k}[i] \exp\{- \eta \frac{\partial  \mathcal{L}(\mathbf{w}_{nn}^{k}, \mathbf{w}_{psl}^{k}, S)}{\partial \mathbf{w}_{psl}^{k}[i]}\}}
    {\sum_{j = 1}^{r} \exp\{-\frac{\partial  \mathcal{L}(\mathbf{w}_{nn}^{k}, \mathbf{w}_{psl}^{k}, S)}{\partial \mathbf{w}_{psl}^{k}[j]}\}} \, , \quad \forall i = 1, \cdots, r\\
\end{align*}
}%
With this update, the symbolic parameter $\mathbf{w}_{psl}$ is guaranteed to satisfy the simplex constraints.

        \section{Dataset Details}
\label{sec:appendix_datasets}

In this section, we provide additional information on the \additionlongname{} and \sudokulongname{} datasets.
Both datasets are generated from the original MNIST image classification dataset introduced by \citenoun{lecun:ieee98}. 
Each MNIST image is a 28x28 matrix consisting of pixel grayscale values normalized to lie in the range $[0,1]$.

\subsection{\additionlongname{}}

The \additionlongname{} task, originally proposed by \citenoun{manhaeve:neurips18}, constructs addition equations using MNIST images with only their summation as a label.
As shown in \figref{fig:mnist-addition}, equations consist of two numbers each comprised of $k$ MNIST images, i.e., \additionlongname{}1 consists of two numbers with one image each ($k=1$) and \additionlongname{}2 consists of two numbers with two images each ($k=2$).
Given two numbers ($2*k$ images), the classification task is to predict the sum.

\textbf{Generation} Addition examples are created by shuffling a list of MNIST images and then partitioning, in order, pairs of numbers.
For example, let the corresponding list of MNIST images be $[\inlinegraphics{images/MNIST-0.png}, \inlinegraphics{images/MNIST-3.png}, \inlinegraphics{images/MNIST-4.png}, \inlinegraphics{images/MNIST-5.png}, \inlinegraphics{images/MNIST-7.png}, \inlinegraphics{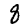}]$.
First this list is shuffled, $[\inlinegraphics{images/MNIST-8.png}, \inlinegraphics{images/MNIST-3.png}, \inlinegraphics{images/MNIST-0.png}, \inlinegraphics{images/MNIST-7.png}, \inlinegraphics{images/MNIST-5.png}, \inlinegraphics{images/MNIST-4.png}]$, and then partitioned into $2*k$ tuples in order.
In this scenario, \additionlongname{}1 creates $3$ addition examples, $\Big[[\inlinegraphics{images/MNIST-8.png}, \inlinegraphics{images/MNIST-3.png}], [\inlinegraphics{images/MNIST-0.png}, \inlinegraphics{images/MNIST-7.png}], [\inlinegraphics{images/MNIST-5.png}, \inlinegraphics{images/MNIST-4.png}]\Big]$.

\begin{figure}[t]
    \centering
    \begin{subfigure}[b]{0.35\textwidth}
        \centering
        \includegraphics[width=\textwidth]{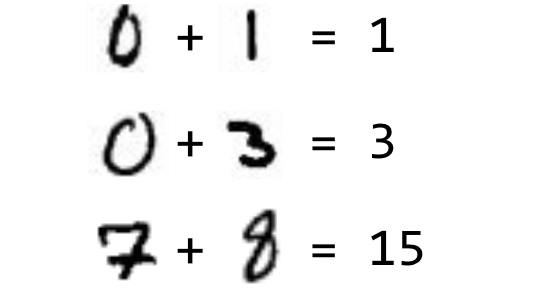}
        \caption{\additionlongname{}1}
    \end{subfigure}
    \begin{subfigure}[b]{0.35\textwidth}
        \centering
        \includegraphics[width=\textwidth]{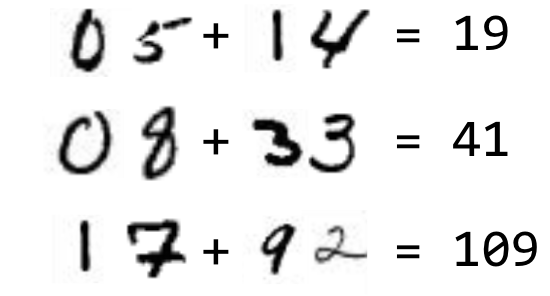}
        \caption{\additionlongname{}2}
    \end{subfigure}
    \caption{
        Example of \additionlongname{}1 and \additionlongname{}2.
    }
    \label{fig:mnist-addition}
\end{figure}

\textbf{Overlap} The process for generating addition examples for overlap variations is the same, but the list of MNIST images contains duplicates.
Specifically, a list of $n \in \{40, 60, 80\}$ unique MNIST images are randomly selected without replacement from the original MNIST train split.
Then, a list of $m \in \{0, n/2, n\}$ images are randomly selected with replacement from these $n$ images.
These two lists are combined to create a final list of MNIST images ($n + m$ images).
This list is used to generate \additionlongname{} examples using the process described above.
This process is then repeated to generate a validation set and then repeated again to generate the test set.
The MNIST images in the test set are pulled from the original MNIST test split to avoid leaking data and $n = 1000$.

\subsection{\sudokulongname{}}

Inspired by the Visual Sudoku problem proposed by \citenoun{wang:icml19}, \citenoun{augustine:nesy22} introduced a novel NeSy task, \textbf{\sudokulongname{}}.
In this task, 4x4 Sudoku puzzles are constructed using unlabeled MNIST images, e.g., \figref{fig:vs-example}.
The model must identify whether a puzzle is correct, i.e., no duplicate digits in any row, column, or square.

\textbf{Generation} Puzzles are created from a list of MNIST images, where this list has an equal representation of each class (e.g., zeroes, ones, twos, and threes).
To create a "correct" puzzle, four images of each class are randomly selected without replacement from this list and arranged in a layout that adheres to the traditional sudoku puzzle rules.
This layout is randomly chosen from all possible correct solutions.
The first image represents the top-left corner, and the final image represents the bottom-right corner of the puzzle.
For example, \figref{fig:vs-example} would be $\{ 1, 2, 4, 3, 4, 3, 1, 2, 2, 4, 3, 1, 3, 1, 2, 4 \}$.

\begin{figure}[t]
    \centering
    \includegraphics[width=0.40 \textwidth]{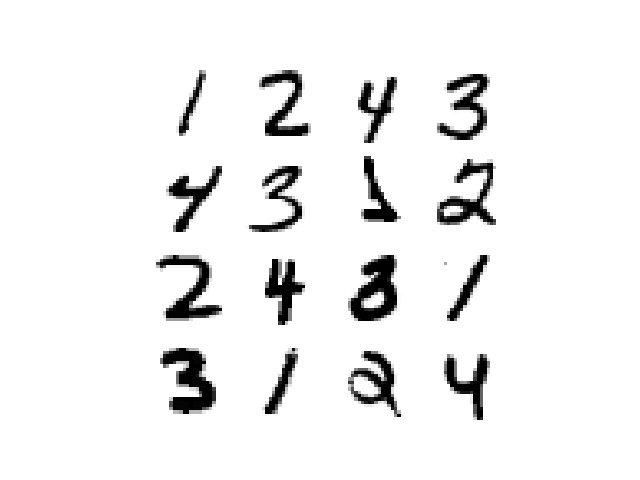}
    \caption{An example of a valid \sudokulongname{} puzzle.}
    \label{fig:vs-example}
\end{figure}

In addition to generating correctly solved Sudoku puzzles, incorrect puzzles are generated.
Instead of randomly creating puzzles and checking if they are correct,
we begin with the correct puzzles and corrupt them.
In this way, we hope to create puzzles that are more subtle and 
closer to the incorrect puzzles that a human may create,
as opposed to randomly generated puzzles that may be obviously incorrect.

The corruptions are done in one of two ways: \textit{replacement} or \textit{substitution}.
A replacement corruption chooses a random cell and replaces it with a random image of another class.
Replacement images are chosen uniformly from the same split.
A substitution corruption randomly chooses two cells in the same puzzle and swaps them.

Each correct puzzle has one corrupted puzzle made from it, resulting in a balanced dataset.
A fair coin is flipped for each puzzle to decide which corruption method will be used.
After each corruption is made, a fair coin is flipped to see if the process continues.
After the complete corruption process, the puzzle is checked to ensure it is not a valid Sudoku puzzle.
If the puzzle is invalid, it is added to the split.
Otherwise, the process is repeated using the same correct puzzle.

\textbf{Overlap} The process for generating puzzle examples for overlap variations is the same, but the list of MNIST images contains duplicates.
Specifically, a list of $n \in \{64, 128, 256\}$ unique MNIST images are randomly selected without replacement from the original MNIST train split, with an equal representation of four classes (zeros, ones, twos, and threes).
Then, a list of $m \in \{0, n, 3.0 \cdot n\}$ images are randomly selected with replacement from these $n$ images, where there is an equal representation of the four class.
These two lists are combined to create a final list of MNIST images ($n + m$ images).
This list is used to generate puzzles using the process described above.
This process is then repeated to generate a validation set and then repeated again to generate the test set.
The MNIST images in the test set are pulled from the original MNIST test split to avoid leaking data and $n = 1000$.

\begin{figure*}[t] 
    \centering
    \noindent\fbox{%
        \begin{minipage}{0.99\hsize}
            \begin{scriptsize}
            \begin{flalign*}
                \hspace{0.3cm} & \textit{\# Digit Sums}&& \\
                & w_{1}: \pslpred{Neural}(\pslarg{Img1}, \pslarg{X}) \psland \pslpred{Neural}(\pslarg{Img2}, \pslarg{Y}) \psland \pslpred{DigitSum}(\pslarg{X}, \pslarg{Y}, \pslarg{Z}) \pslthen \pslpred{Sum}(\pslarg{Img1}, \pslarg{Img2}, \pslarg{Z}) \\
                & w_{2}: \pslneg \pslpred{Neural}(\pslarg{Img1}, \pslarg{X}) \psland \pslpred{Neural}(\pslarg{Img2}, \pslarg{Y}) \psland \pslpred{DigitSum}(\pslarg{X}, \pslarg{Y}, \pslarg{Z}) \pslthen \pslneg \pslpred{Sum}(\pslarg{Img1}, \pslarg{Img2}, \pslarg{Z}) \\
                & w_{3}: \pslpred{Neural}(\pslarg{Img1}, \pslarg{X}) \psland \pslneg \pslpred{Neural}(\pslarg{Img2}, \pslarg{Y}) \psland \pslpred{DigitSum}(\pslarg{X}, \pslarg{Y}, \pslarg{Z}) \pslthen \pslneg \pslpred{Sum}(\pslarg{Img1}, \pslarg{Img2}, \pslarg{Z}) \\[0.3cm]
                & \textit{\# Digit Constraints} \\
                & w_{4}: \pslpred{Neural}(\pslarg{Img1}, +\pslarg{X}) >= \pslpred{Sum}(\pslarg{Img1}, \pslarg{Img2}, \pslarg{Z}) \{\pslarg{X}: \pslpred{PossibleDigits}(\pslarg{X}, \pslarg{Z})\} \\
                & w_{5}: \pslpred{Neural}(\pslarg{Img2}, +\pslarg{X}) >= \pslpred{Sum}(\pslarg{Img1}, \pslarg{Img2}, \pslarg{Z}) \{\pslarg{X}: \pslpred{PossibleDigits}(\pslarg{X}, \pslarg{Z})\} \\[0.3cm]
                & \textit{\# Simplex Constraints} \\
                & \pslpred{Sum}(\pslarg{Img1}, \pslarg{Img2}, +\pslarg{Z}) = 1 . \\
                &
            \end{flalign*}
            \end{scriptsize}
        \end{minipage}
    }
    \caption{\shortname{} \additionlongname{}1 Symbolic Model}
    \label{fig:mnist-addition1-symbolic-model}
\end{figure*}
        \begin{table}[t]
    \centering
    \scriptsize
    \begin{tabular}{cccc}
        \toprule
            Order & Layer       & Parameter       & Value   \\
        \toprule
            \multirow{2}{*}{1} & \multirow{2}{*}{Convolutional} & Kernel Size     & 5       \\
              &                 & Output Channels & 6       \\
        \cmidrule{2-4}
            \multirow{3}{*}{2} & \multirow{3}{*}{Max Pooling}     & Pooling Width   & 2       \\
              &                 & Pooling Height  & 2       \\
              &                 & Activation      & ReLU    \\
        \cmidrule{2-4}
            \multirow{2}{*}{3} & \multirow{2}{*}{Convolutional}   & Kernel Size     & 5       \\
              &                 & Output Channels & 16      \\
        \cmidrule{2-4}
            \multirow{3}{*}{4} & \multirow{3}{*}{Max Pooling}     & Pooling Width   & 2       \\
              &                 & Pooling Height  & 2       \\
              &                 & Activation      & ReLU    \\
        \cmidrule{2-4}
            \multirow{3}{*}{5} & \multirow{3}{*}{Fully Connected} & Input Shape     & 256     \\
              &                 & Output Shape    & 120     \\
              &                 & Activation      & ReLU    \\
        \cmidrule{2-4}
            \multirow{3}{*}{6} & \multirow{3}{*}{Fully Connected} & Input Shape     & 120     \\
              &                 & Output Shape    & 84      \\
              &                 & Activation      & ReLU    \\
        \cmidrule{2-4}
            \multirow{3}{*}{7} & \multirow{3}{*}{Fully Connected} & Input Shape     & 84      \\
              &                 & Output Shape    & 10      \\
              &                 & Activation      & Softmax \\
        \bottomrule
    \end{tabular}

    \caption{
        Neural architecture used in \shortname{} for both \additionlongname{} and \sudokulongname{} experiments.
    }

    \label{tab:mnist-nn-arch}
\end{table}

\section{\shortname \ Models}
\label{sec:appendix_psl_models}
This section provides an overview of the \shortname{} models used in the \textit{Experimental Evaluation}.
The subsequent subsections will examine the symbolic model, neural model, and hyperparameters employed for each setting.

\subsection{\additionlongname{}1}

The \shortname{} model for the \additionlongname{}1 experiment integrates the neural model summarized in \tabref{tab:mnist-nn-arch} with the symbolic model depicted in \figref{fig:mnist-addition1-symbolic-model}.
The symbolic model contains the following predicates:

\begin{itemize}
    \item \textbf{$\pslpred{Neural}(\pslarg{Img}, \pslarg{X})$}
    The $ \pslpred{Neural} $ predicate is the class probability for each image as inferred by the neural network.
    $ \pslarg{Img} $ is MNIST image identifier and $\pslarg{X}$ is a digit class that the image may represent.

    \item \textbf{$ \pslpred{DigitSum}(\pslarg{X}, \pslarg{Y}, \pslarg{Z})$}
    The $\pslpred{DigitSum}$ predicate determines if two digits ($\pslarg{X}$ and $\pslarg{Y}$) sum to a number ($\pslarg{Z}$).
    For example, $\pslpred{DigitSum}(4, 5, 9)$ would return $1$ as $4$ added to $5$ is $9$. Conversely, $\pslpred{DigitSum}(2, 2, 5)$ would return $0$ as $2$ added to $2$ is not $5$.

    \item \textbf{$\pslpred{Sum}(\pslarg{Img1}, \pslarg{Img2}, \pslarg{Z})$}
    The $ \pslpred{Sum} $ predicate is the probability that the digits represented in the images identified by arguments $\pslarg{Img1}$ and $\pslarg{Img2}$ add up to the number identified by the argument $\pslarg{Z}$.
    This predicate instantiates decision variables, i.e., variables from this predicate are not fixed during inference and learning as described in the NeSy EBM, \shortname, and Inference and Learning sections.

    \item \textbf{$\pslpred{PossibleDigits}(\pslarg{X}, \pslarg{Z})$}
    The $ \pslpred{PossibleDigits} $ predicate determines if a digit ($\pslarg{X}$) can be included in a sum that equals a number ($\pslarg{Z}$).
    For example, $\pslpred{PossibleDigits}(9, 0)$ would return $0$ as no positive digit when added to $9$ will equal $0$. Conversely, $\pslpred{PossibleDigits}(9, 17)$ would return $1$ as $8$ added to $9$ equals $17$.
\end{itemize}

The \textit{Digit Sums} rules represents the summation of the two images $\pslarg{Img1}$ and $\pslarg{Img2}$, i.e., if the neural model labels the image id $\pslarg{Img1}$ as digit $\pslarg{X}$ and $\pslarg{Img2}$ as $\pslarg{Y}$ and the digits $\pslarg{X}$ and $\pslarg{Y}$ sum to $\pslarg{Z}$ then the sum of the images must be $\pslarg{Z}$. 

The \textit{Digit Constraints} rules restrict the possible values of the $\pslpred{Sum}$ predicate based on the neural model's prediction. 
For instance, if the neural model predicts that the digit label for image $\pslarg{Img1}$ is 1, then the sum that $\pslarg{Img1}$ is involved in cannot be any less than $1$ or greater than $10$.

\begin{table}[t]
    \centering
    \small

    \begin{tabular}{lll}
        \toprule
            Hyperparameter                    & Tuning Range    & Final Value \\
         \midrule
            Neural Learning Rate              & \{1e-2, 1e-3, 1e-4\}       & 1e-3 \\
            ADMM Max Iterations               & \{50, 100, 500, 1000\}     & 500 \\
        \bottomrule
    \end{tabular}
    
    \caption{
        \shortname{} hyperparameters for the \additionlongname{}1 experiment.
    }
    \label{tab:mnistadd1-hyperparameters}
\end{table}

\textbf{Hyperparameters}  \tabref{tab:mnistadd1-hyperparameters} presents the hyperparameter values and tuning ranges for the \shortname{} \additionlongname{}1 models.
The hyperparameter search was conducted on a single split generated from a list of 600 MNIST images, with the best parameters applied to all data settings.
Any unspecified values were left at their default settings.
The \textit{ADMM Max Iterations} parameter refers to the number of ADMM iterations conducted between each step of gradient descent during the learning process.
The \textit{Neural Learning Rate} parameter refers to the learning rate of the neural model used to predict image labels.

\begin{figure*}[htbptpb] 
    \centering
    \noindent\fbox{%
        \begin{minipage}{0.99\hsize}
            \begin{scriptsize}
            \begin{flalign*}
                \hspace{0.2cm} & \textit{\# Tens Digit Sums}&& \\
                & w_{1}: \pslpred{Neural}(\pslarg{Img1}, \pslarg{X}) \psland \pslpred{Neural}(\pslarg{Img3}, \pslarg{Y}) \psland \pslpred{DigitSum}(\pslarg{X}, \pslarg{Y}, \pslarg{Z}) \pslthen \pslpred{ImageDigitSum}(\pslarg{Img1}, \pslarg{Img3}, \pslarg{Z}) \\
                & w_{2}: \pslneg \pslpred{Neural}(\pslarg{Img1}, \pslarg{X}) \psland \pslpred{Neural}(\pslarg{Img3}, \pslarg{Y}) \psland \pslpred{DigitSum}(\pslarg{X}, \pslarg{Y}, \pslarg{Z}) \pslthen \pslneg \pslpred{ImageDigitSum}(\pslarg{Img1}, \pslarg{Img3}, \pslarg{Z}) \\
                & w_{3}: \pslpred{Neural}(\pslarg{Img1}, \pslarg{X}) \psland \pslneg \pslpred{Neural}(\pslarg{Img3}, \pslarg{Y}) \psland \pslpred{DigitSum}(\pslarg{X}, \pslarg{Y}, \pslarg{Z}) \pslthen \pslneg \pslpred{ImageDigitSum}(\pslarg{Img1}, \pslarg{Img3}, \pslarg{Z}) \\[0.3cm]
                \hspace{0.2cm} & \textit{\# Ones Digit Sums}&& \\
                & w_{4}: \pslpred{Neural}(\pslarg{Img2}, \pslarg{X}) \psland \pslpred{Neural}(\pslarg{Img4}, \pslarg{Y}) \psland \pslpred{DigitSum}(\pslarg{X}, \pslarg{Y}, \pslarg{Z}) \pslthen \pslpred{ImageDigitSum}(\pslarg{Img2}, \pslarg{Img4}, \pslarg{Z}) \\
                & w_{5}: \pslneg \pslpred{Neural}(\pslarg{Img2}, \pslarg{X}) \psland \pslpred{Neural}(\pslarg{Img4}, \pslarg{Y}) \psland \pslpred{DigitSum}(\pslarg{X}, \pslarg{Y}, \pslarg{Z}) \pslthen \pslneg \pslpred{ImageDigitSum}(\pslarg{Img2}, \pslarg{Img4}, \pslarg{Z}) \\
                & w_{6}: \pslpred{Neural}(\pslarg{Img2}, \pslarg{X}) \psland \pslneg \pslpred{Neural}(\pslarg{Img4}, \pslarg{Y}) \psland \pslpred{DigitSum}(\pslarg{X}, \pslarg{Y}, \pslarg{Z}) \pslthen \pslneg \pslpred{ImageDigitSum}(\pslarg{Img2}, \pslarg{Img4}, \pslarg{Z}) \\[0.3cm]
                & \textit{\# Place Digit Sums} \\
                & \pslpred{ImageDigitSum}(\pslarg{Img1}, \pslarg{Img3}, \pslarg{Z10}) \psland \pslpred{ImageDigitSum}(\pslarg{Img2}, \pslarg{Img4}, \pslarg{Z1}) \psland \pslpred{PlaceNumberSum}(\pslarg{Z10}, \pslarg{Z1}, \pslarg{Z}) \\
                & \hspace{2cm} \pslthen \pslpred{Sum}(\pslarg{Img1}, \pslarg{Img2}, \pslarg{Img3}, \pslarg{Img4}, \pslarg{Z}) \\
                & \pslneg \pslpred{ImageDigitSum}(\pslarg{Img1}, \pslarg{Img3}, \pslarg{Z10}) \psland \pslpred{ImageDigitSum}(\pslarg{Img2}, \pslarg{Img4}, \pslarg{Z1}) \psland \pslpred{PlaceNumberSum}(\pslarg{Z10}, \pslarg{Z1}, \pslarg{Z}) \\
                & \hspace{2cm} \pslthen \pslneg \pslpred{Sum}(\pslarg{Img1}, \pslarg{Img2}, \pslarg{Img3}, \pslarg{Img4}, \pslarg{Z}) \\
                & \pslpred{ImageDigitSum}(\pslarg{Img1}, \pslarg{Img3}, \pslarg{Z10}) \psland \pslneg \pslpred{ImageDigitSum}(\pslarg{Img2}, \pslarg{Img4}, \pslarg{Z1}) \psland \pslpred{PlaceNumberSum}(\pslarg{Z10}, \pslarg{Z1}, \pslarg{Z}) \\
                & \hspace{2cm} \pslthen \pslneg \pslpred{Sum}(\pslarg{Img1}, \pslarg{Img2}, \pslarg{Img3}, \pslarg{Img4}, \pslarg{Z}) \\[0.3cm]
                & \textit{\# Tens Digit Constraints} \\
                & w_{7}: \pslpred{Neural}(\pslarg{Img1}, +\pslarg{X}) >= \pslpred{Sum}(\pslarg{Img1}, \pslarg{Img2}, \pslarg{Img3}, \pslarg{Img4}, \pslarg{Z}) \{\pslarg{X}: \pslpred{PossibleTensDigits}(\pslarg{X}, \pslarg{Z})\} \\
                & w_{8}: \pslpred{Neural}(\pslarg{Img3}, +\pslarg{X}) >= \pslpred{Sum}(\pslarg{Img1}, \pslarg{Img2}, \pslarg{Img3}, \pslarg{Img4}, \pslarg{Z}) \{\pslarg{X}: \pslpred{PossibleTensDigits}(\pslarg{X}, \pslarg{Z})\} \\[0.3cm]
                & \textit{\# Ones Digit Constraints} \\
                & w_{9}: \pslpred{Neural}(\pslarg{Img2}, +\pslarg{X}) >= \pslpred{Sum}(\pslarg{Img1}, \pslarg{Img2}, \pslarg{Img3}, \pslarg{Img4}, \pslarg{Z}) \{\pslarg{X}: \pslpred{PossibleOnesDigits}(\pslarg{X}, \pslarg{Z})\} \\
                & w_{10}: \pslpred{Neural}(\pslarg{Img4}, +\pslarg{X}) >= \pslpred{Sum}(\pslarg{Img1}, \pslarg{Img2}, \pslarg{Img3}, \pslarg{Img4}, \pslarg{Z}) \{\pslarg{X}: \pslpred{PossibleOnesDigits}(\pslarg{X}, \pslarg{Z})\} \\[0.3cm]
                & \textit{\# Digit Sum Constraints} \\
                & w_{11}: \pslpred{Neural}(\pslarg{Img1}, +\pslarg{X}) >= \pslpred{ImageDigitSum}(\pslarg{Img1},\pslarg{Img3}, \pslarg{Z}) \{\pslarg{X}: \pslpred{PossibleDigits}(\pslarg{X}, \pslarg{Z})\} \\
                & w_{12}: \pslpred{Neural}(\pslarg{Img3}, +\pslarg{X}) >= \pslpred{ImageDigitSum}(\pslarg{Img1}, \pslarg{Img3}, \pslarg{Z}) \{\pslarg{X}: \pslpred{PossibleDigits}(\pslarg{X}, \pslarg{Z})\} \\
                & w_{13}: \pslpred{Neural}(\pslarg{Img2}, +\pslarg{X}) >= \pslpred{ImageDigitSum}(\pslarg{Img2}, \pslarg{Img4}, \pslarg{Z}) \{\pslarg{X}: \pslpred{PossibleDigits}(\pslarg{X}, \pslarg{Z})\} \\
                & w_{14}: \pslpred{Neural}(\pslarg{Img4}, +\pslarg{X}) >= \pslpred{ImageDigitSum}(\pslarg{Img2}, \pslarg{Img4}, \pslarg{Z}) \{\pslarg{X}: \pslpred{PossibleDigits}(\pslarg{X}, \pslarg{Z})\} \\[0.3cm]
                & \textit{\# Number Sum Constraints} \\
                & \pslpred{ImageDigitSum}(\pslarg{Img1}, \pslarg{Img3}, +\pslarg{X}) >= \pslpred{Sum}(\pslarg{Img1}, \pslarg{Img2}, \pslarg{Img3}, \pslarg{Img4}, \pslarg{Z}) \{\pslarg{X}: \pslpred{PossibleTensSums}(\pslarg{X}, \pslarg{Z})\} \\
                & \pslpred{ImageDigitSum}(\pslarg{Img2}, \pslarg{Img4}, +\pslarg{X}) >= \pslpred{Sum}(\pslarg{Img1}, \pslarg{Img2}, \pslarg{Img3}, \pslarg{Img4}, \pslarg{Z}) \{\pslarg{X}: \pslpred{PossibleOnesSums}(\pslarg{X}, \pslarg{Z})\} \\[0.3cm]
                & \textit{\# Simplex Constraints} \\
                & \pslpred{Sum}(\pslarg{Img1}, \pslarg{Img2}, \pslarg{Img3}, \pslarg{Img4}, +\pslarg{X}) = 1 . \\
                & \pslpred{ImageDigitSum}(\pslarg{Img1}, \pslarg{Img2}, +\pslarg{X}) = 1 .
            \end{flalign*}
            \end{scriptsize}
        \end{minipage}
    }
    \caption{\shortname{} \additionlongname{}2 Symbolic Model}
    \label{fig:mnist-addition2-symbolic-model}
\end{figure*}

\begin{figure*}[ht]
    \centering
    \noindent\fbox{%
        \begin{minipage}{0.99\hsize}
            \begin{scriptsize}
            \begin{flalign*}
                \hspace{0.2cm} & \textit{\# Row Constraint} && \\
                & \pslpred{Neural}(\pslarg{Puzzle}, +\pslarg{X}, \pslarg{Y}, \pslarg{Number}) = 1 . \\[0.3cm]
                & \textit{\# Column Constraint} \\
                & \pslpred{Neural}(\pslarg{Puzzle}, \pslarg{X}, +\pslarg{Y}, \pslarg{Number}) = 1 . \\[0.3cm]
                & \textit{\# Block Constraint} \\
                & \pslpred{Neural}(\pslarg{Puzzle}, ``0", ``0", \pslarg{Number}) + \pslpred{Neural}(\pslarg{Puzzle}, ``0", ``1", \pslarg{Number}) \\
                    & \hspace{2cm} + \pslpred{Neural}(\pslarg{Puzzle}, ``1", ``0", \pslarg{Number}) + \pslpred{Neural}(\pslarg{Puzzle}, ``1", ``1", \pslarg{Number}) = 1 . \\
                & \pslpred{Neural}(\pslarg{Puzzle}, ``2", ``0", \pslarg{Number}) + \pslpred{Neural}(\pslarg{Puzzle}, ``2", ``1", \pslarg{Number}) \\
                    & \hspace{2cm} + \pslpred{Neural}(\pslarg{Puzzle}, ``3", ``0", \pslarg{Number}) + \pslpred{Neural}(\pslarg{Puzzle}, ``3", ``1", \pslarg{Number}) = 1 . \\
                & \pslpred{Neural}(\pslarg{Puzzle}, ``0", ``2", \pslarg{Number}) + \pslpred{Neural}(\pslarg{Puzzle}, ``0", ``3", \pslarg{Number}) \\
                    & \hspace{2cm} + \pslpred{Neural}(\pslarg{Puzzle}, ``1", ``2", \pslarg{Number}) + \pslpred{Neural}(\pslarg{Puzzle}, ``1", ``3", \pslarg{Number}) = 1 . \\
                & \pslpred{Neural}(\pslarg{Puzzle}, ``2", ``2", \pslarg{Number}) + \pslpred{Neural}(\pslarg{Puzzle}, ``2", ``3", \pslarg{Number}) \\
                    & \hspace{2cm} + \pslpred{Neural}(\pslarg{Puzzle}, ``3", ``2", \pslarg{Number}) + \pslpred{Neural}(\pslarg{Puzzle}, ``3", ``3", \pslarg{Number}) = 1 . \\[0.3cm]
                & \textit{\# Pin First Column} \\
                & w_{2}: \pslpred{FirstPuzzle}(\pslarg{Puzzle}, \pslarg{X}, \pslarg{Y}) - \pslpred{Neural}(\pslarg{Puzzle}, \pslarg{X}, \pslarg{Y}) = 0.0
            \end{flalign*}
            \end{scriptsize}
        \end{minipage}
    }
    \caption{\shortname{} \sudokulongname{} Symbolic Model}
    \label{fig:visual-sudoku-symbolic-model}
\end{figure*}

\subsection{\additionlongname{}2}
The \shortname{} model for the \additionlongname{}2 experiment integrates the neural model summarized in \tabref{tab:mnist-nn-arch} with the symbolic model depicted in \figref{fig:mnist-addition2-symbolic-model}.
The symbolic model contains the following predicates:

\begin{itemize}
    \item \textbf{$\pslpred{Neural}(\pslarg{Img}, \pslarg{X})$}
    The $ \pslpred{Neural} $ predicate is the class probability for each image as inferred by the neural network.
    $ \pslarg{Img} $ is MNIST image identifier and $\pslarg{X}$ is a digit class that the image may represent.

    \item \textbf{$ \pslpred{DigitSum}(\pslarg{X}, \pslarg{Y}, \pslarg{Z})$}
    The $\pslpred{DigitSum}$ predicate determines if two digits ($\pslarg{X}$ and $\pslarg{Y}$) sum to a number ($\pslarg{Z}$).
    For example, $\pslpred{DigitSum}(4, 5, 9)$ would return $1$ as $4$ added to $5$ is $9$. Conversely, $\pslpred{DigitSum}(2, 2, 5)$ would return $0$ as $2$ added to $2$ is not $5$.

    \item \textbf{$\pslpred{Sum}(\pslarg{Img1}, \pslarg{Img2}, \pslarg{Img3}, \pslarg{Img4}, \pslarg{Z})$}
    The $ \pslpred{Sum} $ predicate is the probability that the numbers represented in the images identified by arguments $(\pslarg{Img1}, \pslarg{Img2})$ and $(\pslarg{Img3}, \pslarg{Img4})$ add up to the number identified by the argument $\pslarg{Z}$.
    This predicate instantiates decision variables, i.e., variables from this predicate are not fixed during inference and learning as described in the NeSy EBM, \shortname, and Inference and Learning sections.

    \item \textbf{$\pslpred{PossibleTenDigits}(\pslarg{X}, \pslarg{Z})$} $ \pslpred{PossibleTenDigits} $ takes a $0$ or $1$ value representing whether the digit identified by the argument $\pslarg{X}$ is possible when it is in the tens place of a number involved in a sum that totals to the number identified by the argument $\pslarg{Z}$.
    For instance $\pslpred{PossibleTenDigits}(9, 70) = 0$ as no positive number added to a number with a $9$ in the tens place, e.g., $92$, equals $70$, while $\pslpred{PossibleTenDigits}(9, 170) = 1$ as $78$ added to $92$ is $170$.

    \item \textbf{$\pslpred{PossibleOnesDigits}(\pslarg{X}, \pslarg{Z})$}$ \pslpred{PossibleOnesDigits} $ takes a $0$ or $1$ value representing whether the digit identified by the argument $\pslarg{X}$ is possible when it is in the ones place of a number involved in a sum that totals to the number identified by the argument $\pslarg{Z}$.
    For instance $\pslpred{PossibleOnesDigits}(9, 7) = 0$ as no positive number added to a number with a $9$ in the ones place, e.g., $9$, equals $7$ while $\pslpred{PossibleOnesDigits}(9, 170) = 1$ as $71$ added to $99$ is $170$.

    \item \textbf{$\pslpred{ImageDigitSum}(\pslarg{Img1}, \pslarg{Img2}, \pslarg{Z})$}
    The $ \pslpred{ImageDigitSum} $ predicate is the probability that the digits represented in the images specified by $\pslarg{Img1}$ and $\pslarg{Img2}$ will sum up to the number indicated by the argument $\pslarg{Z}$. These variables are considered latent in the NeuPSL model as there are no truth labels for sums of images in the ones or tens places.

    \item \textbf{$\pslpred{PlaceNumberSum}(\pslarg{Z10}, \pslarg{Z1}, \pslarg{Z})$}
    The $ \pslpred{PlaceNumberSum} $ predicate takes a $0$ or $1$ value representing whether the sum of the numbers $\pslarg{Z10}$ and $\pslarg{Z1}$, where $\pslarg{Z10}$ is the sum of digits in the tens place and $\pslarg{Z1}$ is the sum of digits in the one place, adds up to the number $\pslarg{Z}$.
    For instance $\pslpred{PlaceNumberSum}(1, 15, 25)$ is $1$ as $1 \cdot 10 + 15 = 25$.
\end{itemize}

The \textit{Tens Digit Sums} and \textit{Ones Digit Sums} rules compute the sum of two images in the same manner as the \textit{Digit Sums} rules in the \additionlongname{}1 model.
The sum of the digits is captured by the latent variables instantiated by the predicate $\pslpred{ImageDigitSum}$.

The \textit{Place Digit Sums} rules use the value of the $\pslpred{ImageDigitSum}$ variables to infer the sum of the images. 
More specifically, if the $\pslpred{ImageDigitSum}$ of the images in the tens place, $\pslarg{Img1}$ and $\pslarg{Img3})$, is $Z10$, and the $\pslpred{ImageDigitSum}$ of the images in the ones place, $\pslarg{Img2}$ and $\pslarg{Img4})$ is $Z1$, and if according to $\pslpred{PlaceNumberSum}$ the sum of the numbers $Z10$ and $Z1$ is $Z$, then the $\pslpred{Sum}$ of the images must be $Z$.
Notice that these rules are hard constraints as it is always possible and desirable to find values of the $\pslpred{ImageDigitSum}$ and $\pslpred{Sum}$ variables that satisfy these relations.

The \textit{Tens Digit Constraint} rules restrict the possible values of the $\pslpred{Sum}$ predicate based on the neural model's prediction for the digit in the tens place of a number. 
For instance, if the neural model predicts that the digit label for the image $\pslarg{Img1}$ is 1 and $\pslarg{Img1}$ is in the tens place of a number, then the sum that $\pslarg{Img1}$ is involved in cannot be any less than $10$ or greater than $118$.

The \textit{Ones Digit Constraint} rules restrict the possible values of the $\pslpred{Sum}$ predicate based on the neural model's prediction for the digit in the ones place of a number. 
For instance, if the neural model predicts that the digit label for the image $\pslarg{Img2}$ is $5$ and $\pslarg{Img2}$ is in the one place of a number, then the sum that $\pslarg{Img2}$ is involved in cannot be any less than $5$ or greater than $194$.

The \textit{Number Sum Constraint} rules limit the values that $\pslpred{ImageDigitSum}$ and $\pslpred{Sum}$ can take using constraints representing the possible sums in the tens and ones place.
For instance, if the $\pslpred{ImageDigitSum}$ of two images, $\pslarg{Img1}$ and $\pslarg{Img3}$, both in the tens place of two numbers being added, is $17$, then the $\pslpred{Sum}$ cannot be less than $170$ or greater than $188$. 
Furthermore, if the $\pslpred{ImageDigitSum}$ of two images, $\pslarg{Img2}$ and $\pslarg{Img4}$, both in the tens place of two numbers being added, is $17$, then the $\pslpred{Sum}$ cannot be less than $17$ or greater than $197$, and must have a $7$ in the ones place.

\begin{table}[t]
    \centering
    \small

    \begin{tabular}{lll}
        \toprule
            Hyperparameter                    & Tuning Range    & Final Value \\
         \midrule
            Neural Learning Rate              & \{1e-2, 1e-3, 1e-4\}       & 1e-3 \\
            ADMM Max Iterations               & \{50, 100, 500, 1000\}          & 100 \\
        \bottomrule
    \end{tabular}

    \caption{
        \shortname{} hyperparameters for the \additionlongname{}2 experiments.
    }
    \label{tab:mnistadd2-hyperparameters}
\end{table}

\begin{table*}[t]
    \centering

    \begin{tabular}{cclll}
        \toprule
            Number of Images & Number of Puzzles & Hyperparameter & Tuning Range & Final Value   \\
        \midrule
            \multirow{6}{*}{\textasciitilde 64} & \multirow{2}{*}{4} & Neural Learning Rate & \{1e-2, 1e-3, 1e-4\} & 0.001 \\
             &  & ADMM Max Iterations & \{50, 100, 1000\} & 100 \\
             & \multirow{2}{*}{10} & Neural Learning Rate & \{1e-2, 1e-3, 1e-4\} & 0.01 \\
             &  & ADMM Max Iterations & \{50, 100, 1000\} &  50 \\
             & \multirow{2}{*}{20} & Neural Learning Rate & \{1e-2, 1e-3, 1e-4\} & 0.001 \\
             &  & ADMM Max Iterations & \{50, 100, 1000\} & 1000 \\
        \midrule
            \multirow{6}{*}{\textasciitilde 160} & \multirow{2}{*}{10} & Neural Learning Rate & \{1e-2, 1e-3, 1e-4\} & 0.001 \\
             &  & ADMM Max Iterations & \{50, 100, 1000\} & 100 \\
             & \multirow{2}{*}{20} & Neural Learning Rate & \{1e-2, 1e-3, 1e-4\} & 0.001 \\
             &  & ADMM Max Iterations & \{50, 100, 1000\} & 100 \\
             & \multirow{2}{*}{40} & Neural Learning Rate & \{1e-2, 1e-3, 1e-4\} & 0.001 \\
             &  & ADMM Max Iterations & \{50, 100, 1000\} &  100 \\
        \midrule
            \multirow{6}{*}{\textasciitilde 320} & \multirow{2}{*}{20} & Neural Learning Rate & \{1e-2, 1e-3, 1e-4\} & 0.001 \\
             & & ADMM Max Iterations & \{50, 100, 1000\} & 50 \\
             & \multirow{2}{*}{40} & Neural Learning Rate & \{1e-2, 1e-3, 1e-4\} & 0.001 \\
             &  & ADMM Max Iterations & \{50, 100, 1000\} & 50 \\
             & \multirow{2}{*}{80} & Neural Learning Rate & \{1e-2, 1e-3, 1e-4\} & 0.0001 \\
             &  & ADMM Max Iterations & \{50, 100, 1000\} & 100 \\
        \midrule
            \multirow{2}{*}{\textasciitilde 1600} & \multirow{2}{*}{100} & Neural Learning Rate & \{1e-2, 1e-3, 1e-4\} & 0.0001 \\
             &  & ADMM Max Iterations & \{50, 100, 1000\} & 100 \\
        \midrule
            \multirow{2}{*}{\textasciitilde 3200} & \multirow{2}{*}{200} & Neural Learning Rate & \{1e-2, 1e-3, 1e-4\} & 0.01 \\
             &  & ADMM Max Iterations & \{50, 100, 1000\} & 100 \\
        \bottomrule
    \end{tabular}

    \caption{
        \shortname{} hyperparameters for the \sudokulongname{} experiments.
    }

    \label{tab:visual-sudoku-hyperparameters}
\end{table*}

\textbf{Hyperparameters}  \tabref{tab:mnistadd2-hyperparameters} presents the hyperparameter values and tuning ranges for the \shortname{} \additionlongname{}2 models.
The hyperparameter search was conducted on a single split generated from a list of 600 MNIST images, with the best parameters applied to all data settings.
Any unspecified values were left at their default settings.
The \textit{ADMM Max Iterations} parameter refers to the number of ADMM iterations conducted between each step of gradient descent during the learning process.
The \textit{Neural Learning Rate} parameter refers to the learning rate of the neural model used to predict image labels.

\begin{figure*}[t] 
    \centering
    \noindent\fbox{%
        \begin{minipage}{0.99\hsize}
            \begin{scriptsize}
            \begin{flalign*}
                \hspace{0.2cm} & \textit{\# L2 Loss} && \\
                & w_{1}: \pslpred{Neural}(\pslarg{Paper}, \pslarg{Label}) = \pslpred{Category}(\pslarg{Paper}, \pslarg{Label}) \\[0.3cm]
                & \textit{\# Label Propagation} \\
                & w_{2}: \pslpred{Link}(\pslarg{Paper1}, \pslarg{Paper2}) \psland \pslpred{Category}(\pslarg{Paper1}, \pslarg{Label}) \pslthen \pslpred{Category}(\pslarg{Paper2}, \pslarg{Label}) \\[0.3cm]
                & \textit{\# Simplex Constraints} \\
                & \pslpred{Category}(\pslarg{Paper}, +\pslarg{Label}) = 1 .
            \end{flalign*}
            \end{scriptsize}
        \end{minipage}
    }
    \caption{\shortname{} Citation Network Symbolic Model}
    \label{fig:neupsl-citation-network-symbolic-model}
\end{figure*}

\begin{table*}[ht!]
    \centering
    \footnotesize
    \begin{tabular}{ccclll}
        \toprule
            Dataset & Model & Neural/Symbolic & Hyperparameter & Tuning Range & Final Value   \\
        \midrule
            \multirow{16}{*}{Citeseer} & \multirow{8}{*}{\citationlp{}} & \multirow{3}{*}{Neural} & Hidden Layer Size & \{None, 32, 64, 128\} & None \\
             & & & Learning Rate & \{2.0e-0, 1.5e-0, 1.0e-0, 1.0e-1\} & 1.0e-0 \\
             & & & Weight Regularization & \{5.0e-5, 1.0e-5, 5.0e-6, 1.0e-6, 5.0e-7\} & 1.0e-6 \\
             \cmidrule{3-6}
             & & \multirow{5}{*}{Symbolic} & ADMM Step Size & \{0.1, 1.0\} & 1.0 \\
             & & & ADMM Max Iterations & \{25, 100, 1000\} & 25 \\
             & & & Alpha & \{0.0, 0.1\} & 0.0 \\
             & & & Gradient Steps & \{5, 50, 100\} & 50 \\
             & & & Gradient Step Size & \{1.0e-2, 1.0e-3, 1.0e-8\} & 1.0e-8 \\
            \cmidrule{2-6}
             & \multirow{8}{*}{\citationfs{}} & \multirow{3}{*}{Neural} & Hidden Layer Size & \{None, 32, 64, 128\} & None \\
             & & & Learning Rate & \{2.0e-0, 1.5e-0, 1.0e-0, 1.0e-1\} & 1.5e-0 \\
             & & & Weight Regularization & \{5.0e-5, 1.0e-5, 5.0e-6, 1.0e-6, 5.0e-7\} & 1.0e-6 \\
             \cmidrule{3-6}
             & & \multirow{5}{*}{Symbolic} & ADMM Step Size & \{0.01, 0.1, 1.0\} & 1.0 \\
             & & & ADMM Max Iterations & \{25, 100, 1000\} & 1000 \\
             & & & Alpha & \{0.0, 0.1\} & 0.0 \\
             & & & Gradient Steps & \{5, 50, 100\} & 100 \\
             & & & Gradient Step Size & \{1.0e-2, 1.0e-3, 1.0e-8\} & 1.0e-2 \\
        \midrule
            \multirow{16}{*}{Cora} & \multirow{8}{*}{\citationlp{}} & \multirow{3}{*}{Neural} & Hidden Layer Size & \{None, 32, 64, 128\} & None \\
             & & & Learning Rate & \{2.0e-0, 1.5e-0, 1.0e-0, 1.0e-1\} & 1.5e-0 \\
             & & & Weight Regularization & \{5.0e-5, 1.0e-5, 5.0e-6, 1.0e-6, 5.0e-7\} & 5.0e-5 \\
             \cmidrule{3-6}
             & & \multirow{5}{*}{Symbolic} & ADMM Step Size & \{0.01, 0.1, 1.0\} & 1.0 \\
             & & & ADMM Max Iterations & \{25, 100, 1000\} & 25 \\
             & & & Alpha & \{0.0, 0.1\} & 0.0 \\
             & & & Gradient Steps & \{5, 50, 100\} & 50 \\
             & & & Gradient Step Size & \{1.0e-2, 1.0e-3, 1.0e-8\} & 1.0e-8 \\
            \cmidrule{2-6}
             & \multirow{8}{*}{\citationfs{}} & \multirow{3}{*}{Neural} & Hidden Layer Size & \{None, 32, 64, 128\} & None \\
             & & & Learning Rate & \{2.0e-0, 1.5e-0, 1.0e-0, 1.0e-1\} & 1.5e-0 \\
             & & & Weight Regularization & \{5.0e-5, 1.0e-5, 5.0e-6, 1.0e-6, 5.0e-7\} & 5.0e-7 \\
             \cmidrule{3-6}
             & & \multirow{5}{*}{Symbolic} & ADMM Step Size & \{0.01, 0.1, 1.0\} & 1.0 \\
             & & & ADMM Max Iterations & \{25, 100, 1000\} & 1000 \\
             & & & Alpha & \{0.0, 0.1\} & 0.0 \\
             & & & Gradient Steps & \{5, 50, 100\} & 100 \\
             & & & Gradient Step Size & \{1.0e-2, 1.0e-3, 1.0e-8\} & 1.0e-3 \\
        \bottomrule
    \end{tabular}

    \caption{
        \shortname{} hyperparameters for the citation network node classification experiments.
    }

    \label{tab:citation-network-hyperparameters}
\end{table*}

\subsection{Visual-Sudoku-Classification}
The \shortname{} model for the \sudokulongname{} experiment integrates the neural model summarized in \tabref{tab:mnist-nn-arch} with the symbolic model depicted in \figref{fig:visual-sudoku-symbolic-model}.
The symbolic model contains the following predicates:

\begin{itemize}
    \item \textbf{$ \pslpred{Neural}(\pslarg{Puzzle}, \pslarg{X}, \pslarg{Y}, \pslarg{Number})$}
    The $ \pslpred{Neural} $ predicate contains the output class probability for each digit image inferred by the neural network.
    $ \pslarg{Puzzle} $ is sudoku puzzle's identifier, $\pslarg{X}$ and $\pslarg{Y}$ represent the location of image in the puzzle, and $ \pslarg{Number} $ is a digit that image may represent.

    \item \textbf{$ \pslpred{Digit}(\pslarg{Puzzle}, \pslarg{X}, \pslarg{Y}, \pslarg{Number})$}
    The $ \pslpred{Digit} $ predicate has the same arguments as the $ \pslpred{Neural} $ predicate, representing PSL's digit prediction on the image.

    \item \textbf{$ \pslpred{FirstPuzzle, \pslarg{X}, \pslarg{Y}}(\pslarg{Puzzle})$}
    The $ \pslpred{FirstPuzzle} $ predicate pins the values for the first row of the first puzzle to an arbitrary assignment. This is used to force the neural model to learn the correct label representation for easier evaluation.
\end{itemize}

The \textit{Row Constraint}, \textit{Column Constraint}, and \textit{Block Constraint} rules encode the standard Sudoku rules into constraints.
These constraints restrict multiple instances of a digit from appearing in a row, column, or block, respectively.

The \textit{Pin First Column} rules are used to assign arbitrary classes to the first row of a Sudoku puzzle.
The first row of the first correct puzzle from the training set is used to determine this arbitrary label assignment.
By assigning the first row to arbitrary classes, the neural model is provided a starting point for differentiating between the different classes and makes the final evaluation easier.

\textbf{Hyperparameters}  \tabref{tab:visual-sudoku-hyperparameters} presents the hyperparameter values and tuning ranges for the \shortname{} \sudokulongname{} models.
A hyperparameter search was conducted for each data setting on the initial split, with the optimal hyperparameters applied to all subsequent splits.
Any unspecified values were left at their default settings.
The \textit{ADMM Max Iterations} parameter refers to the number of ADMM iterations conducted between each step of gradient descent during the learning process.
The \textit{Neural Learning Rate} parameter refers to the learning rate of the neural model used to predict image labels.

\subsection{Citation Network Node Classification}

The \shortname{} model for the Citation Network Node Classification experiments integrates a single-layered neural model with the symbolic model depicted in \figref{fig:neupsl-citation-network-symbolic-model}.
The single-layer neural model connects the input to a dense-layered output containing a soft-max activation, kernel regularizer, and bias regularizer.
The symbolic model contains the following predicates:

\begin{itemize}
    \item \textbf{$\pslpred{Neural}(\pslarg{Paper}, \pslarg{Label})$}
    The $ \pslpred{Neural} $ predicate contains the output class probability for each paper as inferred by the neural network.
    $ \pslarg{Paper} $ is the identifier and $\pslarg{Label}$ is the category it can take.

    \item \textbf{$\pslpred{Category}(\pslarg{Paper}, \pslarg{Label})$}
    The $ \pslpred{Category} $ predicate has the same arguments as the $ \pslpred{Neural} $ predicate and represents PSL's label prediction on the paper.

    \item \textbf{$\pslpred{Link}(\pslarg{Paper1}, \pslarg{Paper2})$}
    The $ \pslpred{Link} $ predicate denotes whether two papers share a citation link.
\end{itemize}

The \textit{Label Propagation} rule propagates node labels to neighbors.
In this sense, it encodes the idea that papers sharing a citation link are likely to have the same underlying label category.

\textbf{Hyperparameters} \tabref{tab:citation-network-hyperparameters} presents the hyperparameter values and tuning ranges for the \shortname{} citation network node classification models.
A hyperparameter search was conducted for each data setting on the initial split, with the optimal hyperparameters applied to all subsequent splits.
The search process was divided into two distinct stages: a neural hyperparameter search and a symbolic hyperparameter search.
The optimal hyperparameters identified during the neural search were subsequently set during the symbolic search.
All neural models were trained for $250$ epochs utilizing early stopping on the validation set with a patience of $25$.
Final hyperparameter values for LP$_{PSL}$ and Neural$_{PSL}$ are the same as \citationlp{}.
Any unspecified values were left at their default settings.
The \textit{Hidden Layer Size} parameter refers to the size of a single hidden layer, where "None" removes that hidden layer, resulting in a model with only input and output layers.
The \textit{Learning Rate} parameter refers to the learning rate of the neural model.
The \textit{Weight Regularization} parameter adds a kernel and bias regularizer to the hidden layer and output.
The \textit{ADMM Step Size} parameter refers to the initial step size of the ADMM reasoner.
The \textit{ADMM Max Iterations} parameter refers to the number of ADMM iterations conducted between each step of gradient descent during learning.
The \textit{Alpha} is a value that weights the importance of the structural gradient passed back from the symbolic potentials and the gradient with respect to the labels.
The \textit{Gradient Steps} parameter refers to the number of gradient steps taken for joint learning.
The \textit{Gradient Step Size} parameter refers to the step size used in learning the symbolic parameters.
        \begin{table}[t]
    \centering
    \scriptsize

    \begin{tabular}{cccc}
        \toprule
            Order & Layer       & Parameter       & Value          \\
        \toprule
            \multirow{4}{*}{1} & \multirow{4}{*}{Convolutional} & Input Shape     & $28 \times 28$ \\
              &                 & Kernel Size     & 5              \\
              &                 & Output Channels & 6              \\
              &                 & Activation      & ELU            \\
        \cmidrule{2-4}
            \multirow{2}{*}{2} & \multirow{2}{*}{Max Pooling} & Pooling Width   & 2              \\
              &                 & Pooling Height  & 2              \\
        \cmidrule{2-4}
            \multirow{3}{*}{3} & \multirow{3}{*}{Convolutional} & Kernel Size     & 5              \\
              &                 & Output Channels & 16             \\
              &                 & Activation      & ELU            \\
        \cmidrule{2-4}
            \multirow{2}{*}{4} & \multirow{2}{*}{Max Pooling} & Pooling Width   & 2              \\
              &                 & Pooling Height  & 2              \\
        \cmidrule{2-4}
            \multirow{3}{*}{5} & \multirow{3}{*}{Fully Connected} & Input Shape     & 256            \\
              &                 & Output Shape    & 100            \\
              &                 & Activation      & ELU            \\
        \cmidrule{2-4}
            \multirow{3}{*}{6} & \multirow{3}{*}{Concatenation}   & Input Shape     & $2 \times 100$ \\
              &                 & Output Shape    & 200            \\
              &                 & Activation      & ELU            \\
        \cmidrule{2-4}
            \multirow{3}{*}{7} & \multirow{3}{*}{Fully Connected} & Input Shape     & 200            \\
              &                 & Output Shape    & 84             \\
              &                 & Activation      & ELU            \\
        \cmidrule{2-4}
            \multirow{3}{*}{8} & \multirow{3}{*}{Fully Connected} & Input Shape     & 84             \\
              &                 & Output Shape    & 19             \\
              &                 & Activation      & Softmax        \\
        \bottomrule
    \end{tabular}

    \caption{
        Neural architecture for the \additionlongname{}2 CNN baseline \citep{badreddine:ai22}.
    }
    \label{tab:mnist-addition1-cnn-arch}
\end{table}

\begin{table}[t]
    \centering
    \scriptsize
    \begin{tabular}{cccc}
        \toprule
            Order & Layer       & Parameter       & Value          \\
        \toprule
            \multirow{4}{*}{1} & \multirow{4}{*}{Convolutional}   & Input Shape     & $28 \times 28$ \\
              &                 & Kernel Size     & 5              \\
              &                 & Output Channels & 6              \\
              &                 & Activation      & ELU            \\
        \cmidrule{2-4}
            \multirow{2}{*}{2} & \multirow{2}{*}{Max Pooling}     & Pooling Width   & 2              \\
              &                 & Pooling Height  & 2              \\
        \cmidrule{2-4}
            \multirow{3}{*}{3} & \multirow{3}{*}{Convolutional}   & Kernel Size     & 5              \\
              &                 & Output Channels & 16             \\
              &                 & Activation      & ELU            \\
        \cmidrule{2-4}
            \multirow{2}{*}{4} & \multirow{2}{*}{Max Pooling}     & Pooling Width   & 2              \\
              &                 & Pooling Height  & 2              \\
        \cmidrule{2-4}
            \multirow{3}{*}{5} & \multirow{3}{*}{Fully Connected} & Input Shape     & 256            \\
              &                 & Output Shape    & 100            \\
              &                 & Activation      & ELU            \\
        \cmidrule{2-4}
            \multirow{3}{*}{6} & \multirow{3}{*}{Concatenation}   & Input Shape     & $4 \times 100$ \\
              &                 & Output Shape    & 400            \\
              &                 & Activation      & ELU            \\
        \cmidrule{2-4}
            \multirow{3}{*}{7} & \multirow{3}{*}{Fully Connected} & Input Shape     & 400            \\
              &                 & Output Shape    & 128             \\
              &                 & Activation      & ELU            \\
        \cmidrule{2-4}
            \multirow{3}{*}{8} & \multirow{3}{*}{Fully Connected} & Input Shape     & 128             \\
              &                 & Output Shape    & 199             \\
              &                 & Activation      & Softmax        \\
        \bottomrule
    \end{tabular}

    \caption{
        Neural architecture for the \additionlongname{}2 CNN baseline \citep{badreddine:ai22}.
    }
    \label{tab:mnist-addition2-cnn-arch}
\end{table}

\begin{table}[!t]
    \centering
    \resizebox{0.47\textwidth}{!}{
    \begin{tabular}{ccccc}
        \toprule
            \multirow{2}{*}{Model} & Number of & \multirow{2}{*}{Hyperparameter} & \multirow{2}{*}{Tuning Range} & \multirow{2}{*}{Final}  \\
             & Additions & & & \\
        \midrule
            \multirow{6}{*}{\additionlongname{}1} & \multirow{2}{*}{300} & Learning Rate & \{1e-3, 1e-4, 1e-5\} & 1e-3 \\
            & & Batch Size & \{16, 32, 64, 128\} & 32 \\
            & \multirow{2}{*}{3,000} & Learning Rate & \{1e-3, 1e-4, 1e-5\} & 1e-3 \\
            & & Batch Size & \{16, 32, 64, 128\} & 16 \\
            & \multirow{2}{*}{25,000} & Learning Rate & \{1e-3, 1e-4, 1e-5\} & 1e-3 \\
            & & Batch Size & \{16, 32, 64, 128\} & 32 \\
        \midrule
            \multirow{6}{*}{\additionlongname{}2} & \multirow{2}{*}{150} & Learning Rate & \{1e-3, 1e-4, 1e-5\} & 1e-3 \\
            & & Batch Size & \{16, 32, 64, 128\} & 32 \\
            & \multirow{2}{*}{1,500} & Learning Rate & \{1e-3, 1e-4, 1e-5\} & 1e-3 \\
            & & Batch Size & \{16, 32, 64, 128\} & 32 \\
            & \multirow{2}{*}{12,500} & Learning Rate & \{1e-3, 1e-4, 1e-5\} & 1e-3 \\
            & & Batch Size & \{16, 32, 64, 128\} & 64 \\
        \bottomrule
    \end{tabular}
    }
    \caption{
        CNN baseline hyperparameters for the \additionlongname{}1 and \additionlongname{}2 experiments.
    }
    \label{tab:cnn-addition-hyperparameters}
\end{table}

\section{Baseline Models}
\label{sec:appendix_baseline_neural_models}

This section provides additional details of the baseline models used in the \textit{Experimental Evaluation}.
The subsequent subsections will examine the architectural structure and hyperparameters employed for each setting.

\subsection{\additionlongname{}}

The CNN baseline neural models for the \additionlongname{}1 and \additionlongname{}2 experiments are summarized in \tabref{tab:mnist-addition1-cnn-arch} and \tabref{tab:mnist-addition2-cnn-arch} respectively.
These models take as input either two MNIST images (\additionlongname{}1) or four MNIST images (\additionlongname{}1) and output a probability distribution of the resulting sum.
Both models were trained to minimize cross-entropy loss.

\textbf{Hyperparameters} \tabref{tab:cnn-addition-hyperparameters} presents the hyperparameter values and tuning ranges for the baseline \additionlongname{}1 and \additionlongname{}2 models.
A hyperparameter search was conducted for three data sizes on the initial split, with the optimal results applied to all subsequent splits.
All experiments involving overlap utilized the best hyperparameters identified from the \additionlongname{}1 $300$ additions and \additionlongname{}2 $150$ additions searches. 
Any unspecified values were left at their default settings.
The \textit{Batch Size} parameter refers to the number of addition examples per batch of training and evaluation.
The \textit{Learning Rate} parameter refers to the learning rate of the model used to predict.

\begin{table}[!t]
    \centering
    \scriptsize
    \begin{tabular}{cccc}
        \toprule
            Order & Layer       & Parameter       & Value          \\
        \toprule
            \multirow{4}{*}{1} & \multirow{4}{*}{Convolutional}   & Input Shape     & $112 \times 112$ \\
              &                 & Kernel Size     & 3              \\
              &                 & Output Channels & 16             \\
              &                 & Activation      & ReLU            \\
        \cmidrule{2-4}
            \multirow{2}{*}{2} & \multirow{2}{*}{Max Pooling}     & Pooling Width   & 2              \\
              &                 & Pooling Height  & 2              \\
        \cmidrule{2-4}
            \multirow{3}{*}{3} & \multirow{3}{*}{Convolutional}   & Kernel Size     & 3              \\
              &                 & Output Channels & 16             \\
              &                 & Activation      & ReLU            \\
        \cmidrule{2-4}
            \multirow{2}{*}{4} & \multirow{2}{*}{Max Pooling}     & Pooling Width   & 2              \\
              &                 & Pooling Height  & 2              \\
        \cmidrule{2-4}
            \multirow{3}{*}{5} & \multirow{3}{*}{Convolutional}   & Kernel Size     & 3              \\
              &                 & Output Channels & 16             \\
              &                 & Activation      & ReLU            \\
        \cmidrule{2-4}
            \multirow{2}{*}{6} & \multirow{2}{*}{Max Pooling}     & Pooling Width   & 2              \\
              &                 & Pooling Height  & 2              \\
        \cmidrule{2-4}
            \multirow{3}{*}{7} & \multirow{3}{*}{Fully Connected} & Input Shape     & 2304            \\
              &                 & Output Shape    & 256             \\
              &                 & Activation      & ReLU            \\
        \cmidrule{2-4}
            \multirow{3}{*}{8} & \multirow{3}{*}{Fully Connected} & Input Shape     & 256            \\
              &                 & Output Shape    & 256             \\
              &                 & Activation      & ReLU            \\
        \cmidrule{2-4}
            \multirow{3}{*}{9} & \multirow{3}{*}{Fully Connected} & Input Shape     & 256             \\
              &                 & Output Shape    & 128             \\
              &                 & Activation      & ReLU            \\
        \cmidrule{2-4}
            \multirow{3}{*}{10} & \multirow{3}{*}{Fully Connected} & Input Shape     & 128             \\
              &                 & Output Shape    & 1             \\
              &                 & Activation      & Softmax            \\
        \bottomrule
    \end{tabular}
    \caption{
        Neural architecture for the \sudokulongname{} \vsBaselineVisual{} baseline \citep{badreddine:ai22}.
    }
    \label{tab:visual-sudoku-visual-baseline-arch}
\end{table}

\begin{table}[ht!]
    \centering
    \scriptsize
    \begin{tabular}{cccc}
        \toprule
            Order & Layer       & Parameter       & Value          \\
        \toprule
            \multirow{3}{*}{1} & \multirow{3}{*}{Fully Connected} & Input Shape     & 16             \\
              &                 & Output Shape    & 512            \\
              &                 & Activation      & ReLU           \\
        \cmidrule{2-4}
            \multirow{3}{*}{2} & \multirow{3}{*}{Fully Connected} & Input Shape     & 512            \\
              &                 & Output Shape    & 512            \\
              &                 & Activation      & ReLU           \\
        \cmidrule{2-4}
            \multirow{3}{*}{3} & \multirow{3}{*}{Fully Connected} & Input Shape     & 512            \\
              &                 & Output Shape    & 256            \\
              &                 & Activation      & ReLU           \\
        \cmidrule{2-4}
            \multirow{3}{*}{4} & \multirow{3}{*}{Fully Connected} & Input Shape     & 256            \\
              &                 & Output Shape    & 1              \\
              &                 & Activation      & ReLU           \\
        \bottomrule
    \end{tabular}
    
    \caption{
        Neural architecture for the \sudokulongname{} \vsBaselineDigit{} baseline.
    }
    \label{tab:visual-sudoku-digit-baseline-arch}
\end{table}

\subsection{\sudokulongname{}}

The \vsBaselineVisual{} and \vsBaselineDigit{} baseline neural models for the \sudokulongname{} experiments are summarized in \tabref{tab:visual-sudoku-visual-baseline-arch} and \tabref{tab:visual-sudoku-digit-baseline-arch} respectively.
The input to the \vsBaselineVisual{} baseline takes 16 MNIST images as input and produces a probability distribution indicating the likelihood that the images form a correct puzzle.
The input to the \vsBaselineDigit{} baseline takes 16 MNIST image ground truth labels as input and produces a probability distribution indicating the likelihood that the images' labels form a correct puzzle.
Both models were trained to minimize cross-entropy loss.

\begin{table}[t]
    \centering
    \scriptsize

    \begin{tabular}{ccccc}
        \toprule
             \multirow{2}{*}{Model} & Number of & \multirow{2}{*}{Hyperparameter} & \multirow{2}{*}{Tuning Range} & \multirow{2}{*}{Final}  \\
             & Puzzles & & & \\
        \midrule
            \multirow{4}{*}{\vsBaselineVisual{}} & 10 & Learning Rate & \{1e-3, 1e-4, 1e-5\} & 1e-4 \\
            & 20 & Learning Rate & \{1e-3, 1e-4, 1e-5\} & 1e-3 \\
            & 100 & Learning Rate & \{1e-3, 1e-4, 1e-5\} & 1e-2 \\
            & 200 & Learning Rate & \{1e-3, 1e-4, 1e-5\} & 1e-2 \\
        \midrule
            \multirow{4}{*}{\vsBaselineDigit{}} & 10 & Learning Rate & \{1e-3, 1e-4, 1e-5\} & 1e-3 \\
            & 20 & Learning Rate & \{1e-3, 1e-4, 1e-5\} & 1e-2 \\
            & 100 & Learning Rate & \{1e-3, 1e-4, 1e-5\} & 1e-2 \\
            & 200 & Learning Rate & \{1e-3, 1e-4, 1e-5\} & 1e-2 \\
        \bottomrule
    \end{tabular}
    \caption{
        \vsBaselineVisual{} and \vsBaselineDigit{} hyperparameters for the \sudokulongname{} experiment.
    }
    \label{tab:visual-sudoku-baseline-hyperparameters}
\end{table}

\begin{table}[ht]
    \centering
    \scriptsize
    \begin{tabular}{cccc}
        \toprule
            Order & Layer       & Parameters & Value         \\
        \toprule
            \multirow{1}{*}{1} & \multirow{1}{*}{Graph Conv Layer} & Number of Parameters    & 237056             \\
        \cmidrule{2-4}
            \multirow{2}{*}{2} & \multirow{2}{*}{Graph conv Layer} & Number of Parameters     & 390            \\
              &                 & Activation      &  softmax           \\
        \bottomrule
    \end{tabular}
    
    \caption{
        Neural architecture for the citation network node classification GCN model.
    }
    \label{tab:gcn-citation-network-baseline-arch}
\end{table}

\begin{table}[ht!]
    \centering
    \small

    \begin{tabular}{lll}
        \toprule
            Hyperparameter          & Tuning Range & Final Value \\
         \midrule
            Hidden Units        & \{16, 32, 64\} & 64 \\
            Learning Rate       & \{1e-2, 1e-3\} & 1e-3 \\
            Weight Regularizer  & \{1.0e-3, 5.0e-4\} & 1.0e-3 \\
        \bottomrule
    \end{tabular}

    \caption{
        GCN hyperparameters for the citation network node classification experiments.
    }
    \label{tab:gcn-hyperparameters}
\end{table}

\textbf{Hyperparameters} \tabref{tab:visual-sudoku-baseline-hyperparameters} presents the hyperparameter values and tuning ranges for the \vsBaselineVisual{} and \vsBaselineDigit{} baseline neural models.
A hyperparameter search was conducted for each data setting on the initial split, with the optimal hyperparameters applied to all subsequent splits.
Any unspecified values were left at their default settings.
The \textit{Learning Rate} parameter refers to the learning rate of the neural model.

\subsection{Citation Network Node Classification}

As described in the \textit{Experimental Evaluation}, the LP$_{PSL}$ and Neural$_{PSL}$ baseline models represent the distinct symbolic and neural components used in \citationlp{}.
Therefore, the LP$_{PSL}$ model is depicted in \figref{fig:neupsl-citation-network-symbolic-model}, and the Neural$_{PSL}$ model is a single-layered neural model connecting the input to a dense-layered output containing a soft-max activation, kernel regularizer, and bias regularizer.
Hyperparameters were set to the best values found for the \citationlp{} neural hyperparameter search (\tabref{tab:citation-network-hyperparameters}).

The GCN model follows the same architecture proposed by \cite{kipf:iclr17} and is summarized in \tabref{tab:gcn-citation-network-baseline-arch}.
The GCN takes a collection of node identifiers as input and outputs each node's class label.

\textbf{Hyperparameters} \tabref{tab:gcn-hyperparameters} presents the hyperparameter values and tuning ranges for the GCN model.
Each GCN model was trained with 50 percent dropout, a batch size of 1024, and 1000 epochs (utilizing early stopping on the validation set with a patience of 250). 
A hyperparameter search was conducted for each data setting on the initial split, with the optimal hyperparameters applied to all subsequent splits.
Any unspecified values were left at their default settings.

        \begin{table*}[tbh]
    \centering
    \resizebox{\textwidth}{!}{
        \begin{tabular}{c|ccc|ccc}
            \toprule
                \multirow{3}{*}{Method} & \multicolumn{3}{c}{\additionlongname{}1} & \multicolumn{3}{c}{\additionlongname{}2} \\
                & \multicolumn{6}{c}{ 
                       Number of Additions} \\
                           & 300 & 3,000 & 25,000 & 150 & 1,500 & 12,500 \\
             \midrule
                CNN             & 17.16 ± 00.62 & 78.99 ± 01.14 & 96.30 ± 00.30 & 01.31 ± 00.23 & 01.69 ± 00.27 & 23.88 ± 04.32 \\
                \ltnshortname{} & 69.23 ± 15.68 & \textbf{93.90 ± 00.51} & 80.54 ± 23.33 & 02.02 ± 00.97 & 71.79 ± 27.76 & 77.54 ± 35.55 \\
                \dplshortname{} & \textbf{85.61 ± 01.28} & 92.59 ± 01.40 & -\footnotemark[2] & \textbf{71.37 ± 03.90} & \textbf{87.44 ± 02.15} & -\footnotemark[2] \\
                \shortname{}    & 82.58 ± 02.56 & \textbf{93.66 ± 00.32} & \textbf{97.34 ± 00.26} & 56.94 ± 06.33 & \textbf{87.05 ± 01.48} & \textbf{93.91 ± 00.37} \\
            \bottomrule
        \end{tabular}
    }

    \caption{Test set accuracy and standard deviation on \textbf{\additionlongname{}}.
    Results reported here are run and averaged over the same ten splits.}
    \label{tab:MNIST-Addition-n-results}
\end{table*}

\begin{figure*}[t]
    \centering
    \begin{subfigure}[b]{\textwidth}
        \centering
        \includegraphics[width=1.0\textwidth]{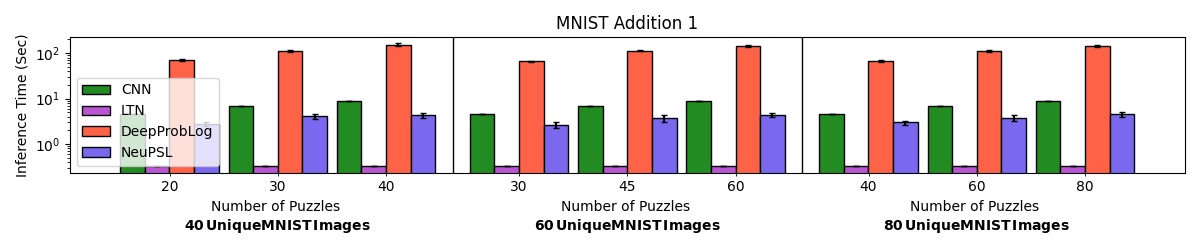}
    \end{subfigure}
    \begin{subfigure}[b]{\textwidth}
        \centering
        \includegraphics[width=1.0\textwidth]{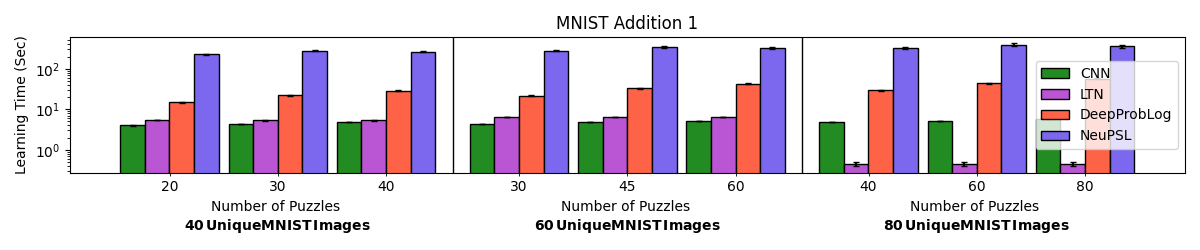}
    \end{subfigure}
    \begin{subfigure}[b]{\textwidth}
        \centering
        \includegraphics[width=1.0\textwidth]{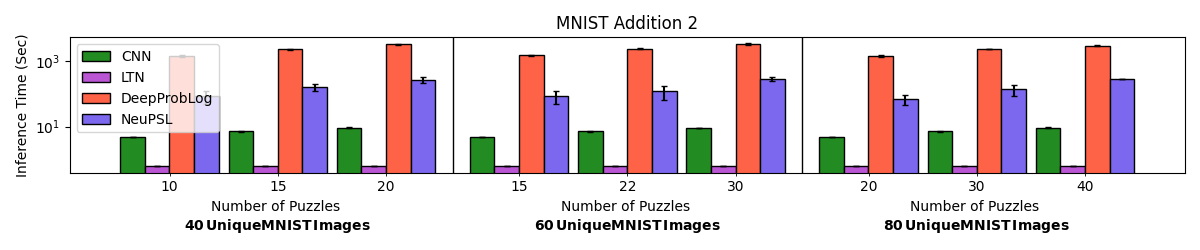}
    \end{subfigure}
    \begin{subfigure}[b]{\textwidth}
        \centering
        \includegraphics[width=1.0\textwidth]{images/mnist-addition-1-learning-runtime-results.png}
    \end{subfigure}
    \caption{
        Inference and learning time for \textbf{\additionlongname{}} experiments presented in \secref{sec:experiments-mnist-addition}.
    }
    \label{fig:mnist_addition_timing_results}
\end{figure*}

\section{Extended Evaluation Details}
\label{sec:extended_evaluation}

This section provides NeSy model details and expands the \textit{Experimental Evaluation} presented earlier on \additionlongname{} and provides inference and learning times for all experiments.

\subsection{NeSy Model Details}

The NeSy methods used in this work, along with their respective publications and implementation codes, are listed below:

\begin{description}
    \item[\dpllongname{} (\dplshortname{}):]
    All \dplshortname{} results use the \dplshortname{} models presented in \citep{manhaeve:ai21}, using default hyperparameters.
    Code was obtained from github.com/ML-KULeuven/deepproblog.

    \item[\dsllongname{} (\dslshortname{}):]
    All \dsllongname{} results use the \dsllongname{} models presented in \citep{winters:arxiv21}, using default hyperparameters.
    Code was obtained from github.com/ML-KULeuven/deepstochlog.

    \item[\ltnlongname{} (\ltnshortname{}):]
    All \ltnshortname{} results use the \ltnshortname{} models presented in \citep{badreddine:ai22}, using default hyperparameters.
    Code was obtained from github.com/logictensornetworks/logictensornetworks.

\end{description}

Licenses for \shortname, \dpllongname{}, \dsllongname{}, are under Apache License 2.0 and \ltnlongname{} are under MIT License.

\subsection{\additionlongname{} Extended Results}
\label{app:mnist_addition_extended_evaluation}

In this section, we conduct an extended analysis of the \additionlongname{} experiment by comparing the performance of \shortname{}, \dpllongname{} (\dplshortname{}) \citep{manhaeve:ai21}, \ltnlongname{} (\ltnshortname{}) \citep{badreddine:ai22} and neural baselines in non-overlap settings with commonly used split sizes in the research community \citep{manhaeve:ai21}.
Ten train splits are generated by randomly selecting, without replacement, $n\in\{600, 6000, 50000\}$ unique MNIST images from the original MNIST train split and converted to MNIST additions as described in the \textit{Datasets} appendix.
This process is then repeated to create validation and test splits, with the test splits being pulled from the original MNIST test split to prevent data leakage and $n=10000$.

\tabref{tab:MNIST-Addition-n-results} shows the average accuracy and standard deviation for \additionlongname{}1 and \additionlongname{}2.\footnote{
    In the largest data setting, there appeared to be an error with DPL, and the results produced were random.
    Rather than present these potentially misleadingly low results, we indicate with ‘-’.
}
The best average accuracy and results within a standard deviation of the best are in bold.
In all but two settings, \shortname{} is either the highest-performing model or within a standard deviation of the highest-performing model.
Moreover, \shortname{} has a markedly lower variance for nearly all training examples in both \additionlongname{} tasks.

\begin{table}[t]
    \centering

    \begin{tabular}{cc|cc}
        \toprule
            \multirow{2}{*}{Setting} & \multirow{2}{*}{Method} & Inference & Learning  \\
            & & (sec) & (sec)  \\
        \midrule
            \multirow{2}{*}{Citeseer} & \citationlp{} & 3.98 ± 0.05 & 29.90 ± 0.82 \\
            & \citationfs{} & 4.23 ± 0.05 & 32.94 ± 0.36 \\
        \midrule
            \multirow{2}{*}{Cora} & \citationlp{} & 4.00 ± 0.31 & 33.41 ± 1.23 \\
            & \citationfs{} & 4.07 ± 0.14 & 36.50 ± 0.53 \\
        \bottomrule
    \end{tabular}
    \caption{
        Inference and learning time for \shortname{} on Citation Network Node Classification experiments presented in \secref{sec:experiments-citation-network-node-classification}.
    }
    \label{tab:citation-inference-learning-runtime}
\end{table}

\begin{table}[t]
    \centering

    \begin{tabular}{cc|cc}
        \toprule
            Unique & \multirow{2}{*}{Puzzles} & Inference & Learning  \\
            Digits & & (sec) & (sec)  \\
        \midrule
            \multirow{3}{*}{64} & 4 & 4.65 ± 0.16 & 43.18 ± 1.35 \\
             & 8 & 6.47 ± 0.19 & 52.56 ± 1.08 \\
             & 16 & 12.56 ± 0.66 & 68.64 ± 0.89 \\
        \midrule
            \multirow{3}{*}{128} & 8 & 4.54 ± 0.07 & 52.45 ± 0.94 \\
             & 16 & 6.48 ± 0.18 & 68.91 ± 1.01 \\
             & 32 & 12.62 ± 0.52 & 102.60 ± 0.90 \\
        \midrule
            \multirow{3}{*}{256} & 8 & 4.67 ± 0.20 & 68.62 ± 1.04 \\
             & 16 & 6.53 ± 0.30 & 102.76 ± 2.05 \\
             & 32 & 12.59 ± 0.53 & 170.66 ± 5.82 \\
        \bottomrule
    \end{tabular}
    \caption{
        Inference and learning time for \shortname{} on Visual Sudoku Puzzle Classification experiments presented in \secref{sec:visual-sudoku}.
    }
    \label{tab:vspc-learning-runtime}
\end{table}

\subsection{Inference and Learning Runtime}
\label{appendix:runtime-results}

\tabref{tab:citation-inference-learning-runtime} summarizes the inference and learning time for \shortname{} on Citation Network Node Classification experiments presented in \secref{sec:experiments-citation-network-node-classification} and \tabref{tab:vspc-learning-runtime} summarizes the inference and learning time for \shortname{} on Visual Sudoku Puzzle Classification experiments presented in \secref{sec:visual-sudoku}.

\figref{fig:mnist_addition_timing_results} summarizes the inference and learning times associated with the \textbf{\additionlongname{}} experiments described in \secref{sec:experiments-mnist-addition}.
When evaluating the performance of the NeSy methods that perform complex symbolic inference (DPL and NeuPSL), a trade-off is observed. 
NeuPSL inference runs an order of magnitude faster than DPL but, surprisingly, takes longer to train on roughly the same number of gradient steps.
This timing difference derives from NeuPSL taking full gradient steps over the entire train dataset while DPL takes batched stochastic gradient steps.
Symbolic inference is a subprocess of NeSy-EBM learning, and DPL performs inference over a single addition, while NeuPSL performs inference over every addition.
An interesting direction for future work is to take batched gradient steps during NeuPSL learning, where the batches contain a set of overlapping additions.

Compared with the CNN and LTN models, DPL and NeuPSL run orders of magnitude slower.
CNN and LTN inference is equivalent to making a feed-forward pass through a neural network.
This will, therefore, be significantly faster than the complex symbolic inference done in \dplshortname{} and \shortname{}, but comes with a decrease in predictive performance.

        \section{Computational Hardware Details}
\label{sec:computational_hardware}

All timing experiments were performed on an Ubuntu 22.04.1 Linux machine with Intel Xeon Processor E5-2630 v4 at 3.10GHz.
    \end{appendix}

\end{document}